\documentclass{article}

\usepackage{amsfonts}
\usepackage[pdftex]{graphicx}
\usepackage{amsthm}
\usepackage{amsmath}
\usepackage{amssymb}
\usepackage{amsbsy}
\usepackage{bm}
\usepackage{subfig}
\usepackage{float}
\usepackage{amstext}
\usepackage{color}
\usepackage{multirow}
\usepackage{boxedminipage}
\usepackage{enumerate}
\usepackage{fullpage}
\usepackage{authblk}

\definecolor{gray128}{RGB}{128,128,128}
\definecolor{gray64}{RGB}{64,64,64}

\DeclareMathOperator*{\argmin}{argmin}

\newtheorem{definition}{Definition}
\newtheorem{fact}{Fact}
\newtheorem{lemma}{Lemma}
\newtheorem{theorem}{Theorem}
\newtheorem{corollary}{Corollary}

\newtheorem{assumption}{Assumption}

\newcommand{\blambda}{{\bm\lambda}}
\newcommand{\bvarphi}{{\bm\varphi}}
\newcommand{\bb}{{\mathbf b}}
\newcommand{\bs}{{\mathbf s}}
\newcommand{\x}{{\mathbf x}}
\newcommand{\y}{{\mathbf y}}
\newcommand{\z}{{\mathbf z}}

\newcommand{\hx}{\hat{x}}

\newcommand{\cO}{{\mathcal O}}

\newcommand{\realset}{{\mathbb R}}

\newcommand{\bitm}{\begin{itemize}}
\newcommand{\eitm}{\end{itemize}}

\newcommand{\benum}{\begin{enumerate}}
\newcommand{\eenum}{\end{enumerate}}

\newcommand{\beq}{\begin{equation}}
\newcommand{\eeq}{\end{equation}}

\newcommand{\beqa}{\begin{eqnarray}}
\newcommand{\eeqa}{\end{eqnarray}}

\newcommand{\bary}{\begin{array}}
\newcommand{\eary}{\end{array}}

\newcommand{\bdefn}{\begin{definition}}
\newcommand{\edefn}{\end{definition}}

\newcommand{\bfct}{\begin{fact}}
\newcommand{\efct}{\end{fact}}

\newcommand{\blem}{\begin{lemma}}
\newcommand{\elem}{\end{lemma}}

\newcommand{\bthm}{\begin{theorem}}
\newcommand{\ethm}{\end{theorem}}

\newcommand{\bcor}{\begin{corollary}}
\newcommand{\ecor}{\end{corollary}}

\newcommand{\bassum}{\begin{assumption}}
\newcommand{\eassum}{\end{assumption}}

\newcommand{\bprf}{\begin{proof}}
\newcommand{\eprf}{\end{proof}}

\usepackage[ruled,lined]{algorithm2e}

\title{Efficient Energy Minimization for Enforcing Statistics}

\author[1]{Yongsub Lim\thanks{yongsub@kaist.ac.kr}}
\author[1]{Kyomin Jung\thanks{kyomin@kaist.edu}}
\author[2]{Pushmeet Kohli\thanks{pkohli@microsoft.com}}
\affil[1]{Korea Advanced Institute of Science and Technology}
\affil[2]{Microsoft Research Cambridge}

\date{}

\begin{document}

\maketitle

\begin{abstract}
Energy minimization algorithms, such as graph cuts, enable the computation of
the MAP solution under certain probabilistic models such as Markov random
fields. However, for many computer vision problems, the MAP solution under the
model is not the ground truth solution. In many problem scenarios, the system
has access to certain statistics of the ground truth. For instance, in image
segmentation, the area and boundary length of the object may be known. In these
cases, we want to estimate the most probable solution that is consistent with
such statistics, i.e., satisfies certain equality or inequality constraints.

The above constrained energy minimization problem is NP-hard in
general, and is usually solved using Linear Programming formulations, which
relax the integrality constraints. This paper proposes a novel method that
finds the discrete optimal solution of such problems by maximizing the
corresponding Lagrangian dual. This method can be applied to any constrained
energy minimization problem whose unconstrained version is polynomial time
solvable, and can handle multiple, equality or inequality, and linear or
non-linear constraints. We demonstrate the efficacy of our method on the
foreground/background image segmentation problem, and show that it produces
impressive segmentation results with less error, and runs more than $20$ times
faster than the state-of-the-art LP relaxation based approaches.
\end{abstract}

\section{Introduction}
Energy minimization has become a key tool for various computer vision problems
such as image segmentation~\cite{BlakeRBPT04,boykoviccv01}, 3D
reconstruction~\cite{sinhaiccv05,Vogiatzis05}, and stereo~\cite{szeliskieccv06}.
It is equivalent to performing Maximum A Posterior (MAP) inference in Markov
Random Fields (MRF), and in general, is NP-hard~\cite{borosdam02}. However,
there are subclasses of energy functions such as submodular functions, and thus
it can be solved in polynomial
time~\cite{borosdam02,kohlicvpr07,kohlicvpr08,kol02}. Also many approximation
algorithms based on belief propagation~\cite{yedidianips01}, tree reweighted
message passing~\cite{kolmogorovpami06}, and graph
cuts~\cite{boykovpami01,Komodakis07} have been studied.

Although sophisticated energy minimization algorithms have enabled researchers
to compute the exact MAP solution under many probabilistic models, they do not
provide the desired solution of image labelling problems.
The reason for this is that most probabilistic models used in computer vision
such as MRFs are mis-specified i.e., the most probable solution under the model
is not the ground truth (desired solution)~\cite{szeliskieccv06,PletscherNKR11}.
This has led researchers to consider more sophisticated
models~\cite{kohlicvpr08,rothcvpr05}, but these too have not been shown to be
free from mis-specification.

For many vision problems, the system may have access to certain statistics of
the ground truth. For instance, in the case of the foreground/background image
segmentation problem, the area and boundary length of the object to be segmented
may be known. Similarly, for 3D voxel segmentation, we may know the volume and
surface area of the object. In these scenarios, researchers require computation
of the most probable solution under the model that is consistent with such
statistics, or in other words, the solution satisfies certain equality or
inequality constraints. Note that the inequality constraint allows researchers,
for instance, to find the most probable segmentation that is bigger or smaller
than a specified value in the image segmentation problem. While imposing such
constraints in the optimization problem results in an improved solution for the
image labelling problem, the resulting optimization problem becomes NP-hard in
general.

The constrained optimization problem (MAP inference problem) is usually solved
using Linear Programming formulations that relax the integrality
constraints~\cite{lempitskyiccv09,nowozincvpr09}. However, this approach suffers
from a number of drawbacks. First, the size of the resulting Linear Program can
be very large, making the algorithm very slow. Second, rounding error introduced
while transforming the fractional solution of the LP to a discrete solution can
be arbitrarily large and may result in solutions having a high energy value.
Lastly, although LP relaxation based methods can deal with linear constraints,
they cannot handle powerful second order or higher order constraints such as the
boundary length constraint described in Section \ref{sec:matstat}.

In this paper, we propose a novel method to solve constrained energy
minimization problems that can handle linear and higher order constraints.
In contrast to previous methods that employ a continuous relaxation of the
problem, we directly exploit its discrete structure. We consider the Lagrangian
dual of the primal constrained problem, and directly compute a discrete solution
corresponding to the dual maximum. It is then guaranteed that such a solution is
the lowest energy solution with its statistics (see Lemma \ref{lemma:1}).
Hence, solutions yielded by our method can be more accurate than those by LP
relaxation based methods with rounding.
Through extensive experiments on the image segmentation problem, we demonstrate
that the proposed method produces solutions with less error, and runs more than
$20$ times faster compared with state-of-the-art continuous relaxation based
methods. Although we show the efficacy of our method via the image
segmentation problem, the method is not limited to a particular problem and can
be applied to various constrained energy minimization problems in computer
vision and machine learning.

\section{Related Work}
Our work has been inspired from a number of recent studies on enforcing
statistics during
inference~\cite{KlodtC11,KolmogorovBR07,LimJK10,Werner08,woodfordiccv09,KohliK10}. These
methods can be divided into two broad categories on the basis of how they
enforce constraints: methods that enforce statistics softly by incorporating
global potentials in probabilistic models, and methods that use statistics as
hard constraints during inference.

Many studies have considered the problem of performing inference in
probabilistic models that contain global potentials that encourage solutions to
have a particular distribution of
labels~\cite{Werner08,woodfordiccv09,KohliK10}. Werner~\cite{Werner08} and
Kohli and Kumar~\cite{kohlicvpr08} proposed inference algorithms that can
handle a global potential that encourages a specific number of variables to be
assigned a particular label. This potential can be used for encouraging
foreground segmentations to be of a particular size, i.e., to cover a specific
number of pixels. Woodford et al.~\cite{woodfordiccv09} considered histogram
distribution preserving potentials for image labelling problems such as image
denoising and texture synthesis and proposed a number of sophisticated methods
for performing inference in models containing such potentials.

The methods discussed above enforce statistics softly, and there is no guarantee
that their solution would have statistics that match the desired statistics. In
fact the statistics can be very different from the desired statistics. A number
of inference algorithms that ensure that solutions have the
desired statistics have been proposed. One of the first constraints used for
image labelling problems was the silhouette constraint, used for the problem of
3D reconstruction~\cite{koleveccv08,sinhaiccv05}. This constraint ensures that a
ray emanating from any silhouette pixel must pass through at least one voxel
that belongs to the `object'. Algorithms that ensure topological
constraints such as connectivity of the object segmentation~\cite{VicenteKR08}
or that the boundary of the object segmentation is close to the sides of a
bounding box~\cite{lempitskyiccv09} have also been proposed. 


Convex relaxation has been widely used to solve a discrete
optimization problem.
Ravikumar and Lafferty~\cite{RavikumarL06} proposed a convex quadratic
relaxation for MAP inference, and it was shown to be more accurate than LP
relaxation and propagation based methods.
Another approach that is closely related to the present work is the method of
Klodt and Cremers~\cite{KlodtC11}, which produces impressive results for
labelling problems such as image segmentation. This method works by solving a
continuous relaxation of the constrained energy minimization problem and
rounding its solution to obtain a discrete solution.
Although their methods can handle multiple constraints, they cannot easily deal
with constraints on second order statistics  (such as the length of the
segmentation boundary) without significantly increasing the computational cost. 
For a more detailed discussion, please refer to Section \ref{sec:dis_con}.

We now discuss the most relevant methods to the present work. Kolmogorov et
al.~\cite{KolmogorovBR07} and Lim et al.~\cite{LimJK10} considered the problem
of minimizing a submodular energy function under the `label count' constraint.
For the foreground/background segmentation problem, the label count constraint
ensures that a specific number of pixels take the foreground label.
Both methods use the parametric mincut algorithm for guiding their search for a
lowest energy solution that satisfies the constraint. The results of these
methods have been shown to improve the accuracy under Hamming error with respect to
the ground truth~\cite{KolmogorovBR07,LimJK10}. However, both methods suffer
from the limitation that they can handle only one constraint at a time.
In contrast, our method can handle multiple, equality or inequality, and linear
or non-linear constraints.


\section{Setup and Preliminaries} \label{sec:setup}

This section provides the notations and definitions used in the paper. It also
provides our formulation for the constrained energy minimization problem and
explains the optimality certificate associated with the solutions of our method.

\subsection{Energy Minimization}

Many computer vision problems are formulated using a random field model defined
over a graph $G=(V,E)$. The energy of the random field is the negative log of
the posterior probability distribution. For image labelling problems with two
labels, such as image segmentation and volumetric segmentation, the energy has
the following form:
\begin{equation} 
	E(\x) = \sum_{c\in C_G} \phi_c(\x_c), 
\end{equation} 
where $\x\in\{0,1\}^n$, $C_G$ is
the set of cliques in $G$, and $\phi_c$ is called a potential and is defined
over a clique $c$. Although minimizing $E$ is NP-hard in general, if $E$ is submodular,
it becomes solvable in polynomial time. Moreover, if an energy function is
defined over cliques of size up to $2$ as 
\begin{equation} \label{eq:e_pair} 
	E(\x) = \sum_{i\in V} \phi_i(x_i) + \sum_{(i,j)\in E} \phi_{ij}(x_i,x_j), 
\end{equation} 
the problem can be efficiently solved by a st-mincut algorithm. Owing to
this property, submodular second order energy functions are widely used in
computer vision for problems such as image segmentation.

\subsection{Constrained Energy Minimization} \label{ssec:const_em}

In this paper, we tackle the problem of minimizing an energy function under
certain pre-specified constraints. This problem is formulated as
\begin{equation} \label{eq:ourineqprob}
	\min_{\x\in\{0,1\}^n} \left\{ E(\x) : b^-_i\leq h_i(\x) \leq b^+_i, 1\leq i\leq
	m \right\},
\end{equation}
where $h_i:\{0,1\}^n \rightarrow \realset$ and $b^-_i,b^+_i\in\realset$. We
denote $(h_1(x),\ldots,h_m(x))$ by $H(x)$. In the context of image segmentation,
$h_i$ may correspond to statistics such as the size and boundary length of an
object. For instance, $h_i = \sum_i x_i$ is the number of pixels that are
assigned the label 1 (foreground).

Before explaining how to solve the inequality problem, we first examine the equality case as follows.
\begin{equation} \label{eq:oureqprob}
	\min_{\x\in\{0,1\}^n} \left\{ E(\x) : H(\x) = \bb \right\}.
\end{equation}
To solve \eqref{eq:oureqprob}, we exploit the Lagrangian dual $D(\blambda)$ of \eqref{eq:oureqprob}.
\begin{equation} \label{eq:eqdual}
	D(\blambda) = \min_{\x\in\{0,1\}^n} L(\x,\blambda),
\end{equation}
where $\blambda\in\realset^m$, and
\begin{equation} \label{eq:lag}
	L(\x,\blambda) = E(\x) + \blambda^T(H(\x) - \bb).
\end{equation}
Note that for any $\bb$, maximizing $D(\blambda)$ gives a lower bound of the
minimum for \eqref{eq:oureqprob}. This leads us to the following known
lemma~\cite{Guignard03}, which, as will become apparent later, provides us with
an optimality certificate that guarantees that our solution is the lowest energy
solution with its statistics.

\begin{lemma}\label{lemma:1}
	Let $\x^*$ be such that $L(\x^*,\blambda^*) = D(\blambda^*)$ for some
	$\blambda^*$, and $\bb^* = H(\x^*)$. Then, $E(\x^*) = \min_{\x\in\{0,1\}^n}
	\{E(\x) : H(\x) = \bb^*\}$.
\end{lemma}
\begin{proof}
	Let $\x$ satisfy that $H(\x) = \bb^*$. Then,
	$\blambda^T(H(\x)-\bb^*)=\blambda^T(H(\x^*)-\bb^*) = 0$ for all $\blambda\in
	\realset^m$. This implies $L(\x^*,\blambda^*) \leq L(\x,\blambda^*)$. Thus,
	from \eqref{eq:lag}, $E(\x^*)\leq E(\x)$.	
\end{proof}

\textcolor{black}{ Note that for a fixed $\x$, $L(\x,\blambda)$ corresponds to an
$(m+1)$-dimensional hyperplane, and when $H(\x)=\bb$, the slope of the
hyperplane becomes $0$. Also note that $D(\blambda)$ is defined by the pointwise
minimum of a finite number of hyperplanes, that is, $D(\blambda)$ is piecewise
linear concave. From these facts, if a hyperplane with slope $0$ contributes to
the shape of the dual $D(\blambda)$, that hyperplane corresponds to a primal
optimal solution and we can find it by maximizing $D(\blambda)$.
Even though the optimal solution whose slope is $0$ does not contribute to the
shape of $D(\blambda)$, we can still compute a (primal) solution at the dual
maximum. Such a (primal) solution $\hx$ has a slope close to $0$, i.e.
$H(\hx)\approx \bb$, and by Lemma \ref{lemma:1}, it is optimal among
$\x\in\{0,1\}^n$ such that $H(x) = H(\hx) \approx \bb$. Thus, it is a good
approximate solution for \eqref{eq:oureqprob}. The algorithm in Section \ref{sec:alg}
essentially computes this solution.}

Now we return to the inequality constrained problem.
An inequality constraint can be handled as an equality one with the insertion of
a {\em slack variable} $y_i\in\realset$. As a first step, we rewrite problem
\eqref{eq:ourineqprob} as follows.
\begin{equation} \label{eq:ourineqprob2}
	\min_\x \left\{ E(\x) : \bb - \bs \leq H(\x) \leq \bb \right\},
\end{equation}
where $\bb=\bb^+$, and $\bs = \bb^+ - \bb^-$. The corresponding equality
problem then becomes
\begin{equation}
	\min_{\x,\y} \left\{ \hat{E}(\x,\y) : H(\x)+\y = \bb \right\},
\end{equation}
where $y_i \in [0,s_i]$, 
and $\hat{E}(\x,\y) = E(\x)$. The following
is the corresponding Lagrangian.
\begin{equation}
	\hat{L}(\x,\y,\blambda) = \hat{E}(\x,\y) + \blambda^T (H(\x) + \y - \bb).
\end{equation}
Let $(\x^*,\y^*)$ be a minimizer of $\hat{L}$ for some $\blambda$; then $y^*_i
= 0$ if $\lambda_i > 0$, $y^*_i = s_i$ if $\lambda_i < 0$, and $y^*_i$ can be any
number in $[0,s_i]$ if $\lambda_i = 0$. Hence, $\y^*$ only depends on
$\blambda$, and the dual is as follows. 
\begin{equation} \label{eq:ineqdual}
	\hat{D}(\blambda) = \min_{\x}\left\{ E(\x) + \blambda^T(H(\x)+ \y^*(\blambda) -\bb)
	\right\}.
\end{equation}
This form is the same as \eqref{eq:eqdual}. Thus, we can handle an inequality
constrained problem as the corresponding equality constrained problem.
We obtain the following Lemma, which shows the optimality of our solution.


\begin{lemma}
	Let $(\x^*,\y^*)$ be such that $\hat{L}(\x^*,\y^*,\blambda^*) =
	\hat{D}(\blambda^*)$ for some $\blambda^*$, and $\bb^* = H(\x^*)+\y^*$. Then
	$E(\x^*) = \min_\x \left\{ E(\x) : \bb^* - \bs\leq H(\x) \leq \bb^* \right\}$.
\end{lemma}
\begin{proof}
	Let $(\x,\y)$ satisfy $H(\x)+\y = \bb^*$. Then, $\blambda^T (H(\x)+\y-\bb^*) =
	\blambda^T (H(\x^*) + \y^* - \bb^*)$ for any $\blambda\in \realset^m$. This
	implies $\hat{L}(\x^*,\y^*,\blambda^*)\leq \hat{L}(\x,\y,\blambda^*)$. Then
	$\hat{E}(\x^*,\y^*)\leq \hat{E}(\x,\y)$, and finally we get $E(\x^*)\leq E(\x)$
	from the definition of $\hat{E}$.
\end{proof}

\section{Algorithm to Maximize the Dual} \label{sec:alg}

We have seen that an energy minimization problem with inequalities can be solved
by considering the corresponding problem with equalities. Thus, in this section
we focus on solving the problem with equalities only. To this end, we employ the
standard cutting plane algorithm~\cite{Guignard03} for maximizing the Lagrangian
dual.
The cutting plane algorithm runs iteratively by computing an $(m+1)$-dimensional
hyperplane consisting of the dual $D$ one by one until finding the maximum.
In the following, we assume the existence of a polynomial time oracle to
compute $D(\blambda)$ for any $\blambda\in S$, where $S=\prod_{i=1}^m
[M_i^-,M_i^+]$ is a large enough set.
Note that when \eqref{eq:e_pair} is submodular and constraints are linear,
$D(\blambda)$ can be efficiently computed by the graph cut algorithm for any
$\blambda\in S$ with $M^-_i=-\infty$ and $M^+_i=\infty$.
In Section \ref{sec:matstat}, we show that this holds for many useful
constraints. We denote the oracle call by the function:
\beq
	\cO(\blambda) = \argmin_{\x\in\{0,1\}^n} L(\x,\blambda).
\eeq

As we mentioned, each $\x\in\{0,1\}^n$ corresponds to a hyperplane
$\{(\blambda,z)\in \realset^{m+1} : \blambda\in \realset^m, v=L(\x,\blambda)\}$.
The cutting plane algorithm works by iteratively finding $\blambda\in S$, for
which the oracle call needs to be made. Let $X=\{\x_1,\ldots,\x_k\}$ be a set of
computed solutions by the oracle call until iteration $k$; we then compute the
maximum $(\blambda^*,z^*)\in \realset^{m+1}$ of the following linear programming
over the variables $(\blambda,z)\in \realset^{m+1}$.
\begin{align} \label{eq:dualprog}
	\max_{z\in \realset} & ~~~~ z \nonumber \\ 
	\text{subject to} & ~~~~ z\leq L(\x,\blambda), \\
	& ~~~~ \x \in X \text{ and } \blambda\in S.
\end{align}
We let $(\blambda^*,z^*)$ be the optimal solution of \eqref{eq:dualprog}.
We subsequently call the oracle to obtain an optimal solution $\x_{k+1}$ for
$\blambda^*$, that is, $\x_{k+1}=\cO(\blambda^*)$. If $z^* =
L(\x_{k+1},\blambda^*)$, the algorithm terminates with the output $\x_{k+1}$;
otherwise, it updates $X=X\cup \{\x_{k+1}\}$, and repeats the procedure.
Note that as new hyperplanes are added over the iterations, the optimal value
of \eqref{eq:dualprog} becomes smaller, and the condition $z^* =
L(\x_{k+1},\blambda^*)$ guarantees the optimality. Algorithm \ref{fig:alg}
summarizes the whole procedure.

\begin{algorithm}[t]
\SetKwRepeat{Repeat}{do}{while}

\LinesNumbered
\DontPrintSemicolon
\BlankLine
\KwIn{Oracle $\cO$}
\KwOut{$\x^*$}
\BlankLine
$X \leftarrow \emptyset$ \;
\Repeat{$L(\x,\blambda^*) < z^*$}{
Compute the maximum $(\blambda^*,z^*)$ by the linear program
\eqref{eq:dualprog}. \; 
$\x \leftarrow \cO(\blambda^*)$, and $X \leftarrow X \cup \{\x\}$ \; }
$\x^* \leftarrow \x$ \;
\caption{Cutting Plane Algorithm} \label{fig:alg}
\end{algorithm}

The running time of the algorithm is dominated by the oracle call and solving
the LP in \eqref{eq:dualprog}. Note that the number of variables, $m$, of
\eqref{eq:dualprog} is small, regardless of that of the primal problem.
Also, the number of constraints is at most $2m+r$, where $r$ is the number of
iterations.
Note that $r$ is bounded by the number of possible $b$ values, 
which is $poly(n)$ for all constraints explained in Section \ref{sec:matstat}. 
For example, $r\leq n$ for the size constraint.
Hence, the running time of our algorithm is polynomial in $n$ for all those constraints.
We also observed that in practice, the number of iterations is 
at most $50$ for almost all cases, thus inducing very fast running time (see Section \ref{sec:exp}).


\section{Constraints for Matching Statistics} \label{sec:matstat}
In this section, we explain how our method can be applied to labelling problems
in vision such as image segmentation. The energy function for the pairwise
random field model for image segmentation can be written as 
\begin{equation}
	\label{eq:e_exp} E(\x) = \sum_{i\in V} \phi_i(x_i) + \sum_{(i,j)\in E}
	C_{ij}|x_i-x_j|,
\end{equation}
where $V$ is the set of all pixels, $E$ is the set of all neighboring edges,
$C_{ij}\geq 0$, and $\phi_i$ encodes the likelihood that each pixel $i$ belongs
to the object or background, and $C_{ij}$ encodes the smoothness of the boundary
of the object. Note that \eqref{eq:e_exp} is submodular since all the second
order terms are submodular. Functions such as \eqref{eq:e_exp} can be
efficiently minimized by solving an equivalent st-mincut
problem~\cite{boykoviccv01}. In the following, we introduce some useful
constraints for image segmentation, which can be handled by our method.

\subsection{Size Constraint (Sz)~\cite{LimJK10}} 
The most natural constraint for segmentation is the size (area for image
segmentation, volume for 3D segmentation) of the object being segmented. This
can be represented as
\beq \tag{Sz}
	b^-_1 \leq \sum_{i\in V} x_i \leq b^+_1.
\eeq

\subsection{Boundary Length Constraint (Br)}
The number of discontinuities in the labelling is a measure of the length of the
object boundary in image segmentation, and the surface area in 3D
reconstruction. These constraints can ensure that the segmentation boundary is
smooth. Hence, we suggest the following boundary length constraint:
\beq \tag{Br} 
	b^-_2 \leq \sum_{(i,j)\in E} |x_i-x_j| \leq b^+_2.
\eeq 
Note that $|x_i-x_j| = 1$ only when pixels $i$ and $j$ contribute to the
boundary of an object. When we use the boundary length constraint, the search
space $S$ for the cutting plane algorithm may be restricted to ensure
submodularity. Note that for some negative $\lambda_{\ell}\in\realset$, where
$\lambda_{\ell}$ is the Lagrangian multiplier for this constraint, the
Lagrangian $L(\x,\lambda_{\ell})$ is no longer submodular over $\x$.
We can ensure submodularity by truncating the search space $S$ such that
$\lambda_{\ell}$ is always larger than $-\min_{i,j} C_{ij}$.
This restriction appears to cause a problem in accuracy, but we observed that
the solution for the parameter $\lambda_{\ell}=0$ (the case of no constraint)
resulted in a speckled segmentation whose boundary length was much larger than
that of the ground truth in all cases. To obtain a shorter boundary length, we
need to search for $\lambda_{\ell}\geq 0$. In other words, the maximum of the
dual is found on a restricted search range,
$\lambda_{\ell}\in[-\min_{ij}C_{ij},\infty]$. Hence, this restriction does not
affect the accuracy in our case.

\subsection{Mean Constraint (Mn)~\cite{KlodtC11}}
The center of the object to be segmented can be easily specified by the user by
drawing a circle roughly containing the object. This information can be used to
define constraints on the mean horizonal and vertical coordinates of the pixels
belonging to the object to be segmented.
Enforcing the mean statistic involves the following constraint:
\beq \tag{Mn} \label{eq:Mn_const}
	b^-_3 \leq \frac{\sum_{i\in V} c_i x_i}{\sum_{i\in V}x_i} \leq b^+_3,
\eeq
where $b^-_3,b^+_3\in\realset^2$, $c_i=(h_i,v_i)^T$, and $h_i$ and $v_i$ are
the horizontal and vertical coordinates of a pixel $i$, respectively. Note that 
\eqref{eq:Mn_const} can be rewritten using a slack variable $y\in \realset^2$ as 
follows.   
\begin{gather}
	\sum_{i\in V} c_ix_i - y\sum_{i\in V} x_i = b^-_3 \sum_{i\in V} x_i, \\
	\sum_{i\in V} (c_i-b^-_3 - y)x_i = 0, \label{eq:translinear}
\end{gather}
where $y\in [0,b^+_{31}-b^-_{31}]\times[0,b^+_{32}-b^-_{32}]$. 
Note that if
we consider the dual $D(\blambda)$ with the mean constraint
\eqref{eq:translinear}, $y$ is determined only by $\blambda$. That is, 
for $i\in\{1,2\}$,
$y_i=b_{3i}^+-b_{3i}^-$ for $\lambda_i>0$, $y_i=0$ for
$\lambda_i<0$, and $y_i$ can be any number for $\lambda_i=0$. Consequently,
\eqref{eq:translinear} can be handled as a linear equality constraint for our
method.


\subsection{Variance Constraint (Vr)~\cite{KlodtC11}}
If we have knowledge of the center of an object, we can also impose a variance
constraint on the distance of the object pixels from the center as follows.
\beq \tag{Vr}
	b^-_4 \leq \frac{\sum_{i\in V} (c_i-\mu)^T(c_i-\mu)x_i }{\sum_{i\in V}x_i} \leq b^+_4, \\
\eeq
where $b^-_4,b^+_4\in \realset^2$, $c_i=(h_i,v_i)^T$, $h_i$ and $v_i$ are
horizontal and vertical coordinates of a pixel $i$, and $\mu\in \realset^2$
denotes the coordinate mean of the object. 
This constraint can be also handled in the same manner as done for Mn.


\subsection{Covariance Constraint (Cv)}
Similarly, a covariance constraint can also be enforced as
\beq \tag{Cv}
	b^-_5 \leq \frac{\sum_{i\in V} (h_i-\mu_h)(v_i-\mu_v)x_i}{\sum_{i\in V}x_i} \leq b^+_5.
\eeq

\subsection{Local Size Constraint (Lsz)}
This is a generalization of the size constraint (Sz). In contrast to the size
constraint, in this constraint, the size of the object being segmented is
locally fixed. For instance, we divide the image into $2\times 2$ subimages, and
impose the size constraint for each subimage. If we could obtain the size
information locally, it would be more powerful than the size constraint.
For instance, we can decompose an object having a complex boundary into parts having 
simpler boundaries and impose the size constraint for each part, respectively. 
In general,
subimages are allowed to overlap each other or may not cover the whole image.
When we enforce the size constraint for $c$ number of subimages, the local size
constraint can be represented as follows. For $1\leq j\leq c$,
\beq \tag{Lsz}
	b_{6j}^- \leq \sum_{i\in V_j} x_i \leq b_{6j}^+.
\eeq

\begin{figure}
\centering
	\includegraphics[width=0.9\columnwidth]{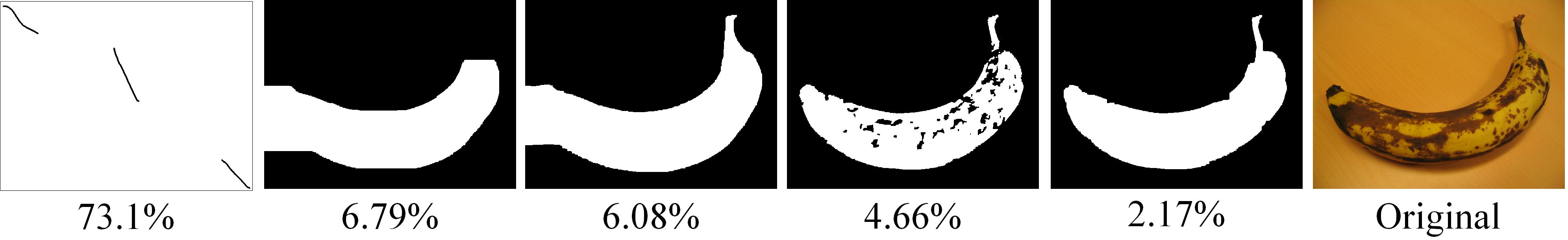}
	\caption{Example of a sequence of segmentations computed by the cutting plane
	algorithm to maximize the dual. The size (Sz) and boundary length constraints
	(Br) are considered in this figure. The segmentations shown in the figure are
	labeled with the percentage of erroneous pixels (inconsistent with the ground
	truth labels) in the segmentation. The image has a size of $640\times
	480$.}\label{fig:optpath}
\end{figure}

\begin{figure*}[t]
\centering
	\includegraphics[width=0.9\columnwidth]{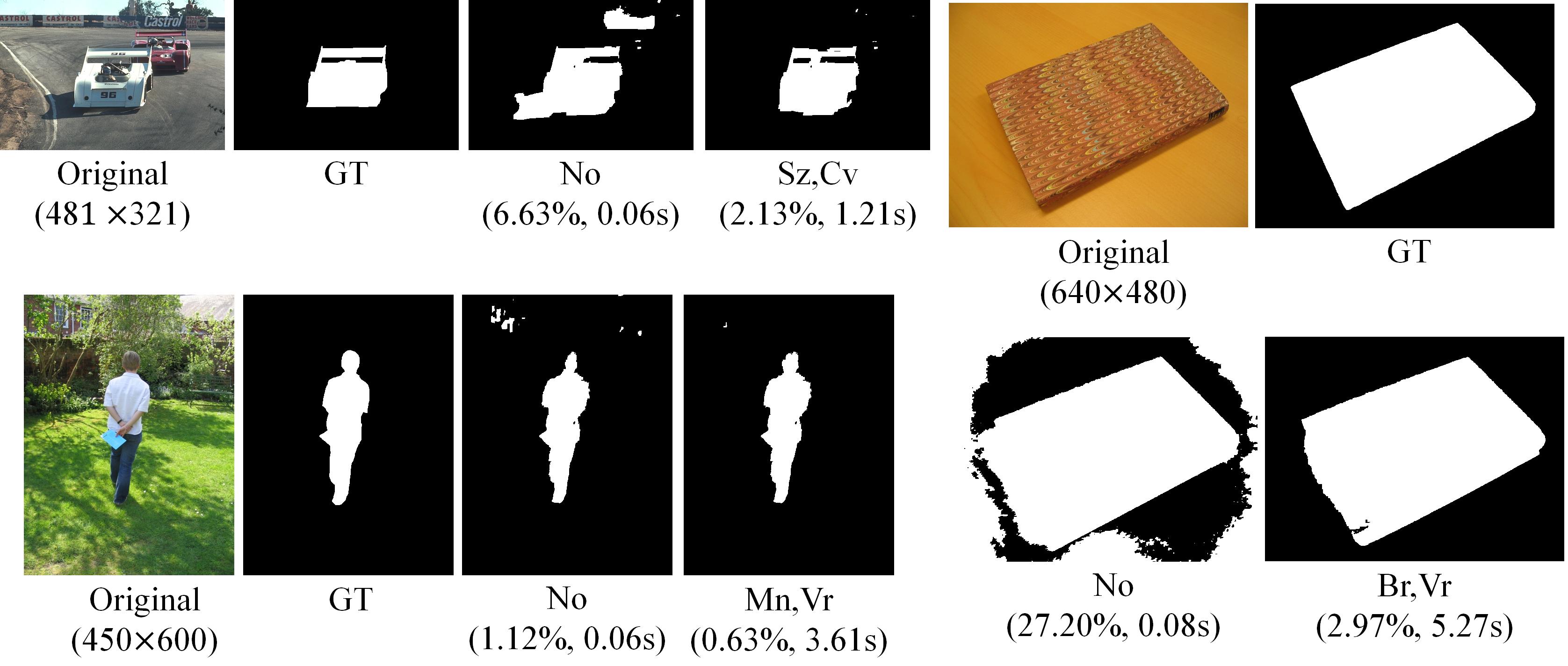}
	\caption{Effect of constraints on the segmentation results. Each segmentation
	is labelled by the percentage of erroneous pixels with respect to the ground
	truth, and the computation time.} \label{fig:impv}
\end{figure*}

All the constraints above, except Br, are linear in $\x$. Note that for a linear
constraint, we can adopt any search range $S$ in the cutting plane algorithm
since it does not break submodularity. Now, we can use any combination of these
constraints with an appropriate search range. \figurename~\ref{fig:optpath}
shows a sequence of  segmentations produced by our method for maximizing the
dual. \figurename~\ref{fig:impv} shows some examples of segmentations improved
by imposing constraints.

\section{Experiments} \label{sec:exp}
We now evaluate the performance of our algorithm on the image segmentation
problem using images from the GrabCut dataset~\cite{RotherKB04}. The first
experiment checks the accuracy of the segmentation results obtained with
different combinations of constraints. We also compare the results of our method
with those obtained by the continuous relaxation based methods
in~\cite{lempitskyiccv09,RavikumarL06}. 
To measure the quality of segmentations, we
compute the percentage of erroneously labelled pixels (with respect to the
ground truth(GT)). We use ER to denote this percentage error.
The experiment environment is a server with 2.8GHZ Quad-Core Intel i7 930 and
24G memory. 

\begin{figure}[t]
\centering
	\includegraphics[width=0.9\columnwidth]{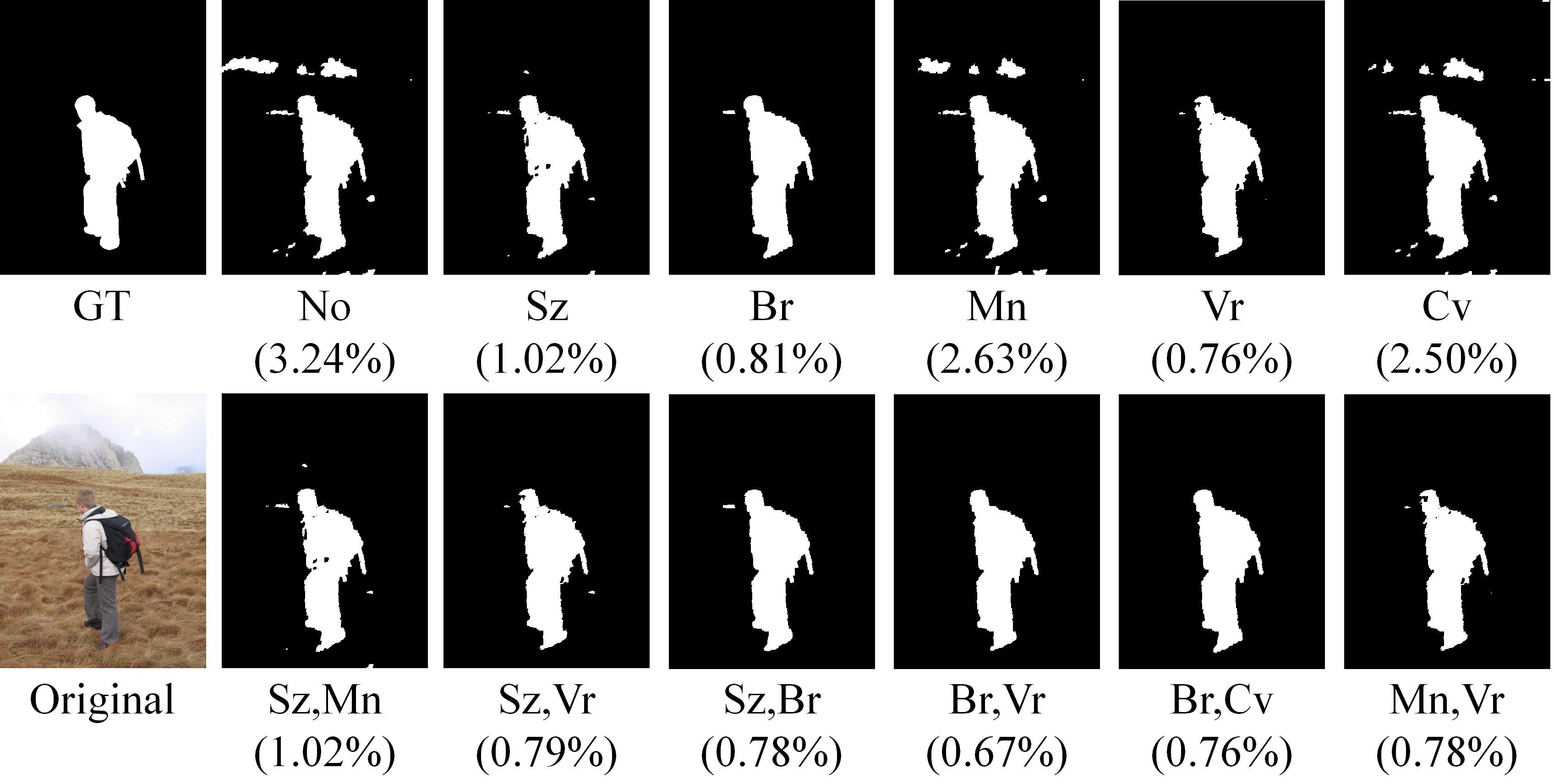}
	\caption{Comparison of segmentations computed with different combination of constraints. The reported number is the percentage error (ER) with respect to the ground truth segmentations. The size of the image is $450\times 600$.}	\label{fig:eg_seg}
\end{figure}

\subsection{Effect of Imposing Constraints}

\figurename~\ref{fig:eg_seg} shows segmentation results corresponding to various
combinations of constraints. We find that the variance constraint (Vr) is quite
effective: the results with constraint combinations including the variance
constraint have less error compared with the other results. Imposing the
boundary length constraint (Br) affected the smoothness of segmentation. For
instance, the ER values of segmentations with boundary length/covariance
constraints (Br, Cv) and variance constraint (Vr) are nearly the same, but the
former is much smoother than the latter. \tablename~\ref{tab:ineq_result} shows
quantitative results of our method. We assumed that statistics of the ground
truth such as area, mean, and variance are known, and set inequality gaps in
intervals of $-10\sim10\%$ and $-5\sim5\%$ from the ground truth values.
We considered various combinations of constraints. As expected, imposition of
any constraint improves the segmentation accuracy on average. From the table, we
can see that the size (Sz) and variance constraints (Vr) are especially powerful
for image segmentation. Note that the variance constraint (Vr) alone results in
a highly accurate segmentation. We speculate that this is because the variance
constraint causes the segmentation result to be rounder, and reduces noise
fragmentations. We provide more segmentation results by our method with various
constraint combinations in Appendix.

\begin{table}[t]
\small
\renewcommand{\arraystretch}{1.3}
\caption{Percentage pixel error (ER) and the running time for results obtained
with our method. Each number was averaged over $14$ images. Rows are sorted in
ER for an inequality gap interval of $-10 \sim 10\%$.} \label{tab:ineq_result}
\centering
\begin{tabular}{cccccc|cccccc} \hline 
	\multicolumn{2}{c}{\multirow{2}{15mm}{}} & \multicolumn{2}{c}{$-10\sim 10\%$} & \multicolumn{2}{c|}{$-5\sim 5\%$} &
	\multicolumn{2}{c}{\multirow{2}{15mm}{}} & \multicolumn{2}{c}{$-10\sim 10\%$} & \multicolumn{2}{c}{$-5\sim 5\%$} \\ \cline{3-6} \cline{9-12}
	&& ER$(\%)$ & Time(s) & ER$(\%)$ & Time(s) &&& ER$(\%)$ & Time(s) & ER$(\%)$ & Time(s) \\ \hline
	\multicolumn{1}{c}{$1$} & Sz,Br,Vr & 2.13 & 6.42 & 1.83 & 7.04 &
	\multicolumn{1}{c}{$12$} & Sz,Mn,Cv & 3.27 & 2.81 & 2.73 & 2.95 \\
	
	\multicolumn{1}{c}{$2$} & Sz,Vr & 2.30 & 2.05 & 2.00 & 2.22 &
	\multicolumn{1}{c}{$13$} & Sz,Mn & 3.29 & 1.64 & 2.72 & 1.84 \\		
	
	\multicolumn{1}{c}{$3$} & Vr,Cv & 2.30 & 2.41 & 2.12 & 2.40 &
	\multicolumn{1}{c}{$14$} & Sz & 3.49 & 0.65 & 3.10 & 0.69 \\
	
	\multicolumn{1}{c}{$4$} & Br,Vr & 2.33 & 3.92 & 2.16 & 4.87 &
	\multicolumn{1}{c}{$15$} & Br,Mn,Cv & 6.39 & 10.43 & 6.32 & 11.51 \\
	
	\multicolumn{1}{c}{$5$} & Vr & 2.41 & 1.27 & 2.28 & 1.37 &
	\multicolumn{1}{c}{$16$} & Br,Cv & 6.50 & 6.11 & 6.61 & 6.11 \\
	
	\multicolumn{1}{c}{$6$} & Br,Vr,Cv & 2.41 & 7.24 & 2.13 & 6.67 &
	\multicolumn{1}{c}{$17$} & Br,Mn & 6.57 & 6.12 & 6.22 & 7.34 \\
	
	\multicolumn{1}{c}{$7$} & Mn,Vr & 2.53 & 2.27 & 2.28 & 2.60 &
	\multicolumn{1}{c}{$18$} & Mn,Cv & 6.70 & 1.85 & 6.43 & 2.09 \\
	
	\multicolumn{1}{c}{$8$} & Sz,Br,Mn & 3.04 & 6.46 & 2.52 & 9.28 &
	\multicolumn{1}{c}{$19$} & Cv & 6.76 & 0.81 & 6.76 & 0.81 \\
	
	\multicolumn{1}{c}{$9$} & Sz,Br,Cv & 3.05 & 5.53 & 2.60 & 6.08 &		
	\multicolumn{1}{c}{$20$} & Mn & 6.79 & 0.93 & 6.12 & 1.07 \\
	
	\multicolumn{1}{c}{$10$} & Sz,Br & 3.19 & 3.57 & 2.40 & 3.39 &
	\multicolumn{1}{c}{$21$} & Br & 6.96 & 4.49 & 7.04 & 4.57 \\
	
	\multicolumn{1}{c}{$11$} & Sz,Cv & 3.32 & 1.54 & 3.00 & 1.60 &
	\multicolumn{1}{c}{$22$} & No & 7.05 & 0.06 & 7.05 & 0.06 \\

	\hline	\multicolumn{6}{c}{}\\
\end{tabular}
\end{table}

\subsection{Comparison with Continuous Relaxation}

We compared the proposed method with two continuous relaxation based methods.

\subsubsection{Linear Relaxation (LR)~\cite{lempitskyiccv09}}
The first one is to relax the problem to a linear programming using
additional variables $z\in\realset^{|E|}$, as follows.
\beq
\begin{split}
	\min_{\x\in[0,1]^n, \z\in\realset^{|E|}}& \bar{E}(\x,\z) := \bvarphi'  \x +
	\sum_{(i,j)\in E} C_{ij} z_{ij} \\
	\text{subject to}\quad & z_{ij}\geq x_i-x_j, \text{ and} \\
	&z_{ij}\geq x_j-x_i, \text{ for every $(i,j)\in E$},
\end{split}
\eeq
where $\varphi_i = \phi_i(1) - \phi_i(0)$. This LP relaxation can handle all
constraints described in Section \ref{sec:matstat}, but cannot deal with
both-side inequality or an equality boundary length constraint (Br), because
each of these breaks the convexity of the feasible region. Also, it is not
reasonable to use $\sum_{(i,j)\in E} z_{ij}$ as the value of the boundary length constraint (Br)
because this value is not exactly the same as $\sum_{(i,j)\in E} |x_i-x_j|$.
Hence, the boundary length constraint (Br) cannot be properly handled by the LR.

\subsubsection{Quadratic Relaxation (QR)~\cite{RavikumarL06}}  
For this relaxation, we first write our energy function as follows.
\beq
	E(\x) = \bvarphi' x_i + \sum_{(i,j)\in E} C_{ij}(x_i-x_j)^2.
\eeq
Note that $(x_i-x_j)^2 = |x_i-x_j|$ since $\x\in \{0,1\}^n$. The relaxed
problem can then be formulated as a convex quadratic programming.
\beq
	\min_{\x\in [0,1]^n} \tilde{E}(x) = \x'H\x + \bvarphi'\x,
\eeq
where $H_{ij} = -C_{ij}$ for $i\neq j$, and $H_{ii}= \sum_{j:(i,j)\in E}
C_{ij}$. This quadratic relaxation can handle all constraints described in
Section \ref{sec:matstat}, except the boundary length constraint (Br). As in LR,
both-side inequality or an equality boundary length constraint (Br) breaks the
convexity of the feasible region.

\subsubsection{Results}

\begin{figure}[t]
\centering
	\includegraphics[width=0.9\columnwidth]{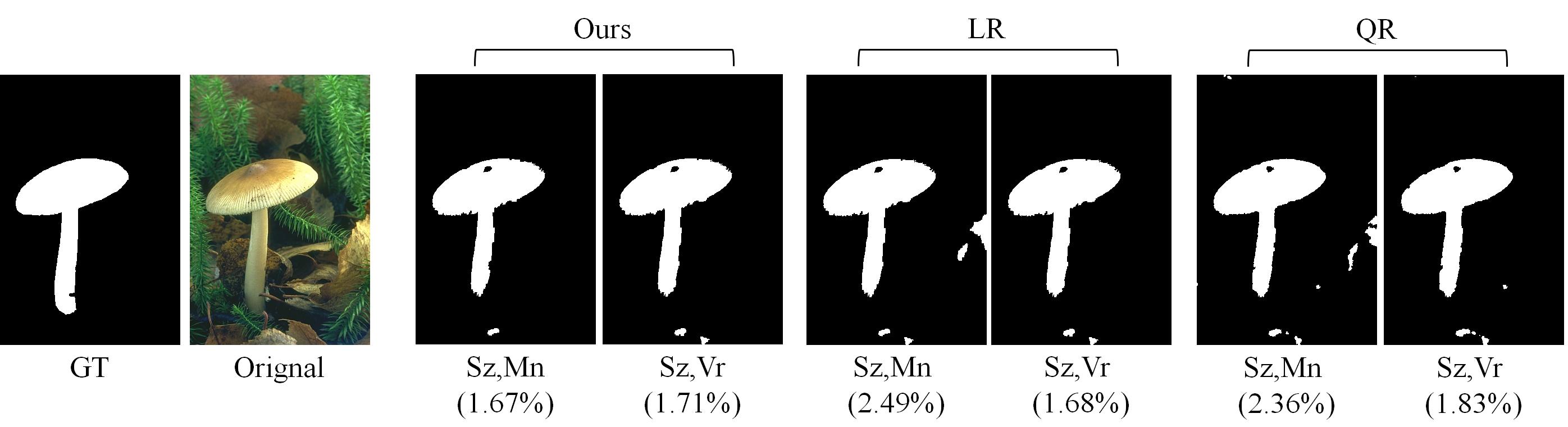}
	\caption{Examples of segmentations computed by LR, QR, and our method.
	Inequality gaps for each constraint combination are $-10\sim10\%$. The size of
	the image is $321\times 481$.}	 \label{fig:eg_cr_seg}
\end{figure}


We implemented the LR and QR based methods using CPLEX Optimizer with MATLAB
interface, and implemented our method by combining MATLAB and C++ codes.
\tablename~\ref{tab:comp_relax} summarizes the comparison between the proposed
approach and the relaxation based methods outlined above.

For all constraint combinations, our method produces segmentations with less
percentage pixel error (ER) compared to those obtained by LR and QR. Further,
our method is extremely fast: $200\sim300$ times faster than QR and $20\sim100$
times faster than LR. \figurename~\ref{fig:eg_cr_seg} shows some examples of
segmentation by the three methods. We provide more segmentation results to
compare our method with LR and QR in Appendix.

\begin{figure}
\centering
	\includegraphics[width=0.9\columnwidth]{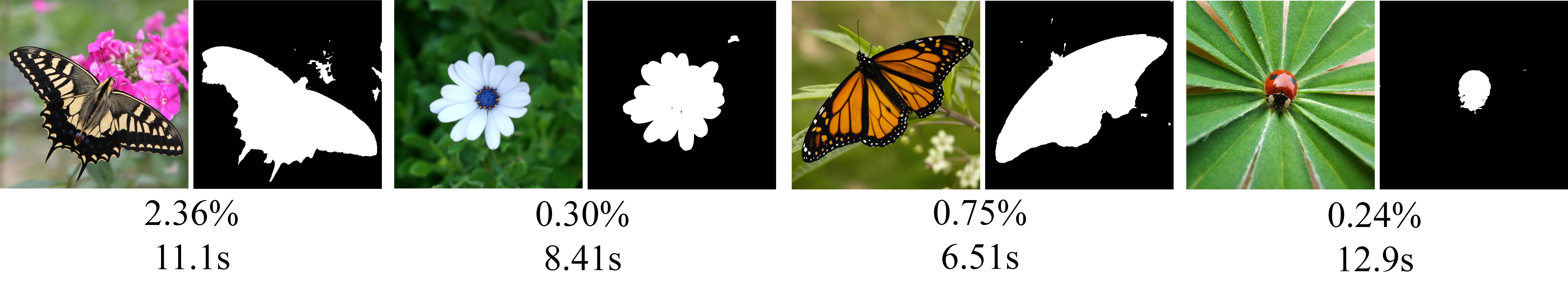}
	\caption{Example of segmentations by our method. Each original image has a
	size of $1024\times 1024$.}
	\label{fig:large_exp}
\end{figure}

\begin{table*}[t]
\renewcommand{\arraystretch}{1.3}
\caption{Comparison of the segmentation errors of our method with those
obtained from the continuous relaxation and rounding based methods. The
inequality gaps used for this experiment were $-10\sim10\%$, and $-5\sim 5\%$.
The percentage error is averaged over $6$ images of size $321\times
481$. Bold indicates the best error and time in each row.}
\label{tab:comp_relax}
\centering
\scriptsize
\begin{tabular}{ccccccc|cccccc} \hline
	&\multicolumn{6}{c|}{$\pm 10\%$}&\multicolumn{6}{c}{$\pm 5\%$} \\ \cline{2-13}
	\multirow{2}{8mm}{} & \multicolumn{2}{c}{Our Method} & \multicolumn{2}{c}{LR} &
	\multicolumn{2}{c|}{QR} &
	 \multicolumn{2}{c}{Our Method} & \multicolumn{2}{c}{LR} &
	 \multicolumn{2}{c}{QR}
	 \\
	 \cline{2-13}
	
	&ER$(\%)$ & Time(s) & ER$(\%)$ & Time(s) & ER$(\%)$ & Time(s) &
	ER$(\%)$ & Time(s) & ER$(\%)$ & Time(s) & ER$(\%)$ & Time(s)
	\\ \hline
	
	Sz & {\bf 3.49} & {\bf 0.65} & 4.22 & 39.7 & 4.70 & 280 &
	{\bf 3.10} & {\bf 0.69} & 3.84 & 68.4 & 4.48 & 316 \\
	
	Mn & 6.79 & {\bf 0.93} & {\bf 5.90} & 21.7 & 6.61 & 320 &
	6.12 & {\bf 1.07} & {\bf 5.67} & 26.8 & 6.23 & 336 \\
	
	Vr & {\bf 2.41} & {\bf 1.27} & 3.04 & 51.5 & 3.10 & 430 &
	{\bf 2.28} & {\bf 1.37} & 2.90 & 50.7 & 2.86 & 425  \\
	
	Sz,Mn & {\bf 3.29} & {\bf 1.64} & 4.09 & 57.7 & 4.50 & 307 &
	{\bf 2.72} & {\bf 1.84} & 3.46 & 83.6 & 3.95 & 349 \\
	
	Sz,Vr & {\bf 2.30} & {\bf 2.05} & 3.02 & 64.9 & 3.06 & 445 &
	{\bf 2.00} & {\bf 2.22} & 2.73 & 77.2 & 2.64 & 461\\
	
	Mn,Vr & {\bf 2.53} & {\bf 2.27} & 3.04 & 62.3 & 3.10 & 487 &
	{\bf 2.28} & {\bf 2.60} & 2.90 & 59.4 & 2.86 & 477 \\
	\hline \multicolumn{5}{c}{}
\end{tabular}
\end{table*}

Recall that our method runs by calling the oracle and solving the LP in
\eqref{eq:dualprog} at each iteration. Since we consider the pixel grid graph,
the oracle call, which involves minimizing a submodular pseudo-boolean function
using the graph cut algorithm, takes very little time. The number of variables
constituting the LP in \eqref{eq:dualprog} for our approach is small.
Furthermore, for all constraint combinations in \tablename~\ref{tab:comp_relax},
the total number of iterations does not exceed $50$ (average: $27$, std.:
$11.7$), making the number of constraints of the LP small. The small size of the
LP enables our method to run much faster than LR, which has to solves a
relatively large linear programming problem.

The results of our experiments comparing the computational complexity of the
three methods (ours, LR, QR) are  shown in \figurename~\ref{fig:sc}. We
used four $1024\times 1024$ images, and each image was scaled down to $8$ small
images of size $2^{i/2}\times 2^{i/2}$, $12\leq i\leq 19$. Each data point
denotes the average running time of the specified method over $4$ images having
the same size specified on the x-axis values. It can be seen that our methods is
substantially faster than the state-of-the-art LR and QR based approaches.
\figurename~\ref{fig:large_exp} shows segmentation results yielded by our method
for $4$ original images of size $1024\times 1024$.

\begin{figure}[t]
\centering
	\includegraphics[width=0.8\columnwidth]{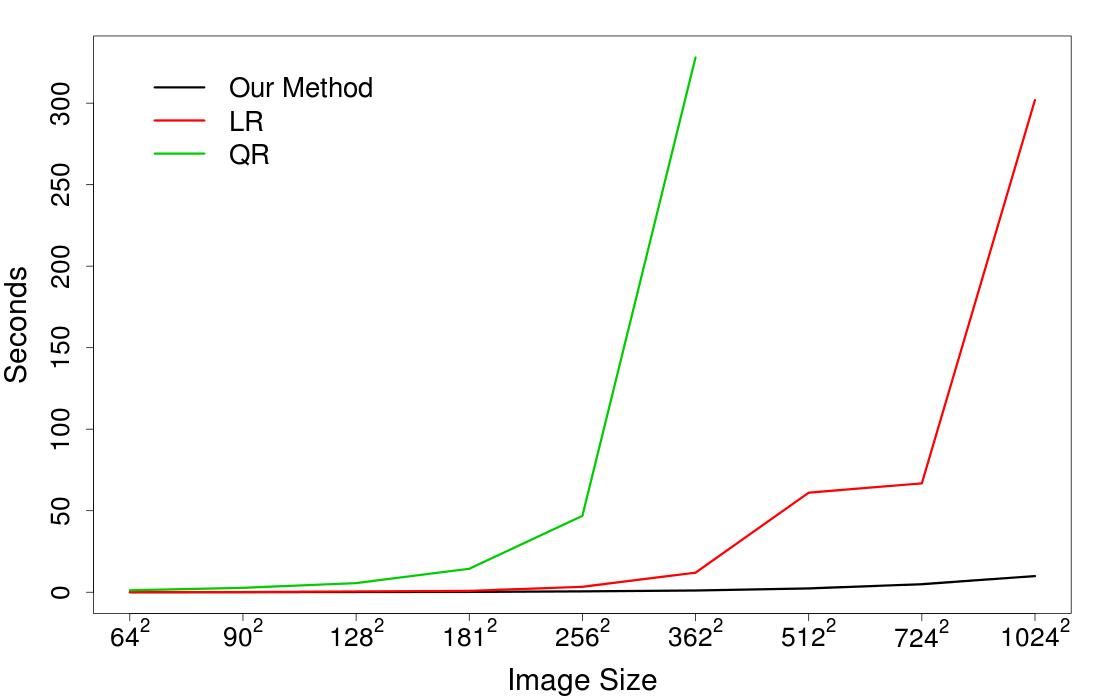}
	\caption{Comparison of scalability between our method and two relaxation based
	methods.}
	\label{fig:sc}
\end{figure}

\tablename~\ref{tab:err} shows how solutions by the three methods satisfy given
constraints. Here, we used two measures: $\text{ER}_\text{a}$ and
$\text{ER}_\text{b}$.
$\text{ER}_\text{a}$ denotes the percentage of constraints that are unsatisfied.
For instance, if the size (Sz) and variance constraints (Vr) are considered and
one of them is not satisfied, then $\text{ER}_\text{b}$ is counted as $0.5$.
$\text{ER}_\text{b}$ represents the deviation of the constraint values from the
inequality gap, thus providing a measure of the degree of satisfaction. For
instance, let the inequality gap for Sz be $[a,b]$ and the size of a computed
solution be $c$; $\text{ER}_\text{b}$ is then counted by $\min(|c-a|,|c-b|) /
((a+b)/2)$ if $c\notin[a,b]$ and counted by $0$ if $c\in[a,b]$. We observed that
solutions computed by three methods satisfy the constraints or are very close to
the constraints gap in general.

\begin{table}[t]
\renewcommand{\arraystretch}{1.3}
\caption{Comparison between our method and two relaxation based methods on
whether computed solutions satisfy constraints or not.} \label{tab:err}
\centering
\begin{tabular}{ccccccc} \hline
	& \multicolumn{2}{c}{Our Method} & \multicolumn{2}{c}{LR} &
	\multicolumn{2}{c}{QR}
	\\
	\cline{2-7}
	& $\text{ER}_\text{a}$ & $\text{ER}_\text{b}$ & $\text{ER}_\text{a}$ &
	$\text{ER}_\text{b}$ & $\text{ER}_\text{a}$ & $\text{ER}_\text{b}$ \\ \cline{2-7}
	$\pm 10\%$ & 14.35 & 0.88 & 22.69 & 0.40 & 7.2 & 0.06 \\
	$\pm 5\%$ & 29.63 & 1.12 & 23.84 & 0.54 & 26.00 & 0.42 \\ \hline
\end{tabular}
\end{table}

\section{Discussion and Conclusions} \label{sec:dis_con}

We have proposed a novel approach to constrained energy
minimization. Our method is efficient and can handle a very large class of
constraints including those that enforce second order statistics, which cannot
be achieved by using previously proposed methods.
To handle such second order constraints, the boundary length constraint in our
paper, we use a restricted search region in which the dual is computable. In our
cases, the dual maximum is always found in that restricted search region which
means that it does not affect the accuracy of our method.
Although our approach can be used with continuous relaxations based methods
\cite{lempitskyiccv09,RavikumarL06}, it would require repeated solutions of the
relaxed problem to find the dual maximum and would be much slower than our
method as shown in Section \ref{sec:exp}. We evaluated our method on image
segmentation with constraints described in Section \ref{sec:matstat}. Our method
produced segmentations, with less error, in much faster time compared to
state-of-the-art continuous relaxation based methods.



\bibliographystyle{abbrv}
\bibliography{ref}

\begin{thebibliography}{10}

\bibitem{BlakeRBPT04}
A.~Blake, C.~Rother, M.~Brown, P.~P{\'e}rez, and P.~Torr.
\newblock Interactive image segmentation using an adaptive gmmrf model.
\newblock In {\em ECCV}, 2004.

\bibitem{borosdam02}
E.~Boros and P.~Hammer.
\newblock Pseudo-boolean optimization.
\newblock {\em Discrete Applied Mathematics}, 2002.

\bibitem{boykoviccv01}
Y.~Boykov and M.~Jolly.
\newblock Interactive graph cuts for optimal boundary and region segmentation
  of objects in {N-D} images.
\newblock In {\em ICCV}, 2001.

\bibitem{boykovpami01}
Y.~Boykov, O.~Veksler, and R.~Zabih.
\newblock Fast approximate energy minimization via graph cuts.
\newblock {\em PAMI}, 2001.

\bibitem{Guignard03}
M.~Guignard.
\newblock Lagrangean relaxation.
\newblock {\em TOP}, 11, 2003.

\bibitem{KlodtC11}
A.~Klodt and D.~Cremers.
\newblock A convex framework for image segmentation with moment constraints.
\newblock In {\em ICCV}, 2011.

\bibitem{kohlicvpr07}
P.~Kohli, M.~Kumar, and P.~Torr.
\newblock $p^3$ and beyond: Solving energies with higher order cliques.
\newblock In {\em CVPR}, 2007.

\bibitem{KohliK10}
P.~Kohli and M.~P. Kumar.
\newblock Energy minimization for linear envelope mrfs.
\newblock In {\em CVPR}, pages 1863--1870, 2010.

\bibitem{kohlicvpr08}
P.~Kohli, L.~Ladicky, and P.~Torr.
\newblock Robust higher order potentials for enforcing label consistency.
\newblock In {\em CVPR}, 2008.

\bibitem{koleveccv08}
K.~Kolev and D.~Cremers.
\newblock Integration of multiview stereo and silhouettes via convex
  functionals on convex domains.
\newblock In {\em ECCV}, 2008.

\bibitem{kolmogorovpami06}
V.~Kolmogorov.
\newblock Convergent tree-reweighted message passing for energy minimization.
\newblock {\em PAMI}, 2006.

\bibitem{KolmogorovBR07}
V.~Kolmogorov, Y.~Boykov, and C.~Rother.
\newblock Application of parametric maxflow in computer vision.
\newblock In {\em ICCV}, 2007.

\bibitem{kol02}
V.~Kolmogorov and R.~Zabih.
\newblock What energy functions can be minimized using graph cuts?
\newblock In {\em ECCV}, 2002.

\bibitem{Komodakis07}
N.~Komodakis, N.~Paragios, and G.~Tziritas.
\newblock Mrf optimization via dual decomposition: message-passing revisited.
\newblock In {\em ICCV}, 2007.

\bibitem{lempitskyiccv09}
V.~Lempitsky, P.~Kohli, C.~Rother, and T.~Sharp.
\newblock Image segmentation with a bounding box prior.
\newblock In {\em ICCV}, 2009.

\bibitem{LimJK10}
Y.~Lim, K.~Jung, and P.~Kohli.
\newblock Energy minimization under constraints on label counts.
\newblock In {\em ECCV}, 2010.

\bibitem{nowozincvpr09}
S.~Nowozin and C.~Lampert.
\newblock Global connectivity potentials for random field models.
\newblock In {\em CVPR}, 2009.

\bibitem{PletscherNKR11}
P.~Pletscher, S.~Nowozin, P.~Kohli, and C.~Rother.
\newblock Putting map back on the map.
\newblock In {\em DAGM-Symposium}, pages 111--121, 2011.

\bibitem{RavikumarL06}
P.~D. Ravikumar and J.~D. Lafferty.
\newblock Quadratic programming relaxations for metric labeling and markov
  random field map estimation.
\newblock In {\em ICML}, 2006.

\bibitem{rothcvpr05}
S.~Roth and M.~Black.
\newblock Fields of experts: A framework for learning image priors.
\newblock In {\em CVPR}, 2005.

\bibitem{RotherKB04}
C.~Rother, V.~Kolmogorov, and A.~Blake.
\newblock "grabcut": interactive foreground extraction using iterated graph
  cuts.
\newblock {\em ACM Trans. Graph.}, 2004.

\bibitem{sinhaiccv05}
S.~Sinha and M.~Pollefeys.
\newblock Multi-view reconstruction using photo-consistency and exact
  silhouette constraints: A maximum-flow formulation.
\newblock In {\em ICCV}, 2005.

\bibitem{szeliskieccv06}
R.~Szeliski, R.~Zabih, D.~Scharstein, O.~Veksler, V.~Kolmogorov, A.~Agarwala,
  M.~Tappen, and C.~Rother.
\newblock A comparative study of energy minimization methods for {M}arkov
  random fields.
\newblock In {\em ECCV}, 2006.

\bibitem{VicenteKR08}
S.~Vicente, V.~Kolmogorov, and C.~Rother.
\newblock Graph cut based image segmentation with connectivity priors.
\newblock In {\em CVPR}, 2008.

\bibitem{Vogiatzis05}
G.~Vogiatzis, P.~Torr, and R.~Cipolla.
\newblock Multi-view stereo via volumetric graph-cuts.
\newblock In {\em CVPR}, 2005.

\bibitem{Werner08}
T.~Werner.
\newblock High-arity interactions, polyhedral relaxations, and cutting plane
  algorithm for soft contraint optimisation (map-mrf).
\newblock In {\em CVPR}, 2008.

\bibitem{woodfordiccv09}
O.~Woodford, C.~Rother, and V.~Kolmogorov.
\newblock A global perspective on {MAP} inference for low-level vision.
\newblock In {\em ICCV}, 2009.

\bibitem{yedidianips01}
J.~Yedidia, W.~Freeman, and Y.~Weiss.
\newblock Generalized belief propagation.
\newblock In {\em NIPS}, 2001.

\end{thebibliography}

\newpage

\setcounter{section}{0}

\renewcommand{\thesection}{\Alph{section}}

\noindent
{\bf\LARGE Appendix}

\bigskip
In Appendix, we show various segmentation results.
Each segmentation is labelled by pixel-wise error and running time, and obtained
with inequality gap $\pm 5\%$.


\section{Segmentation Results of Our Method with Various Constraint Combinations}
 In the following, segmentation results by our method with various constraint
 combinations are shown.

\begin{center}
\includegraphics[width=.47\columnwidth]{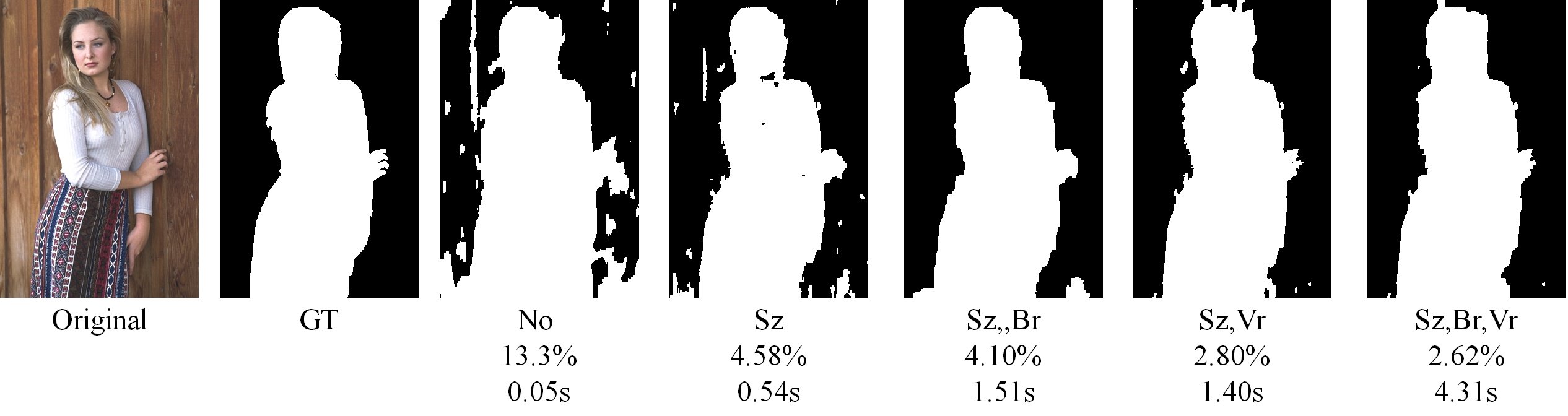} \qquad
\includegraphics[width=.47\columnwidth]{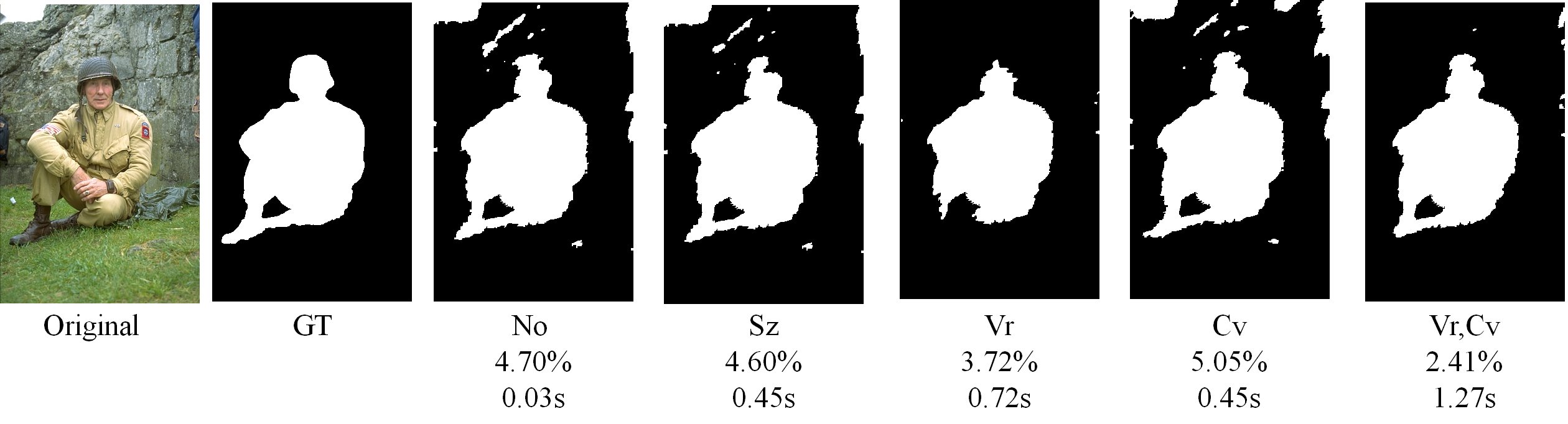}
\includegraphics[width=.47\columnwidth]{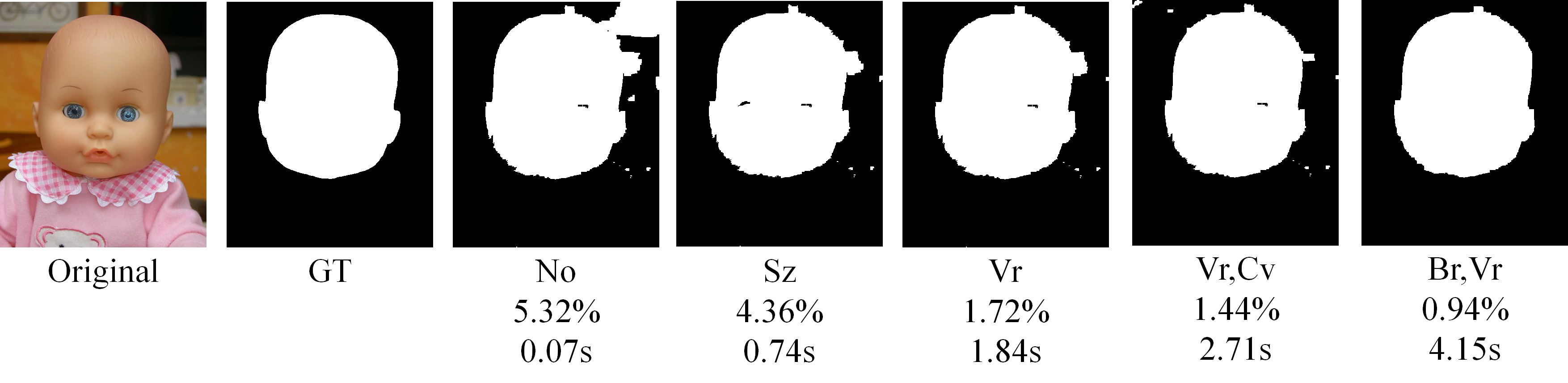} \qquad
\includegraphics[width=.47\columnwidth]{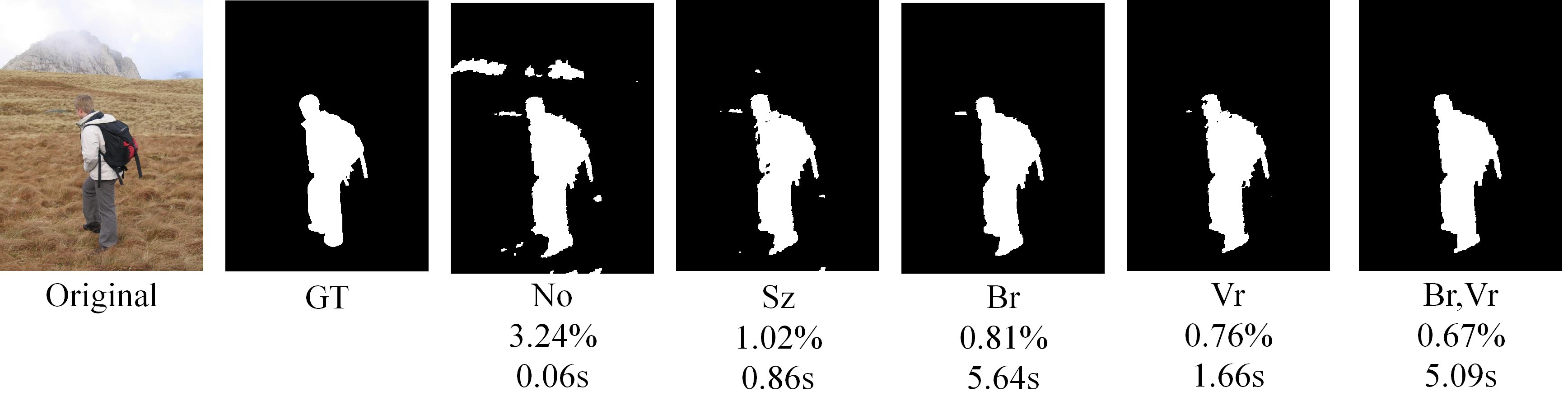}
\includegraphics[width=.47\columnwidth]{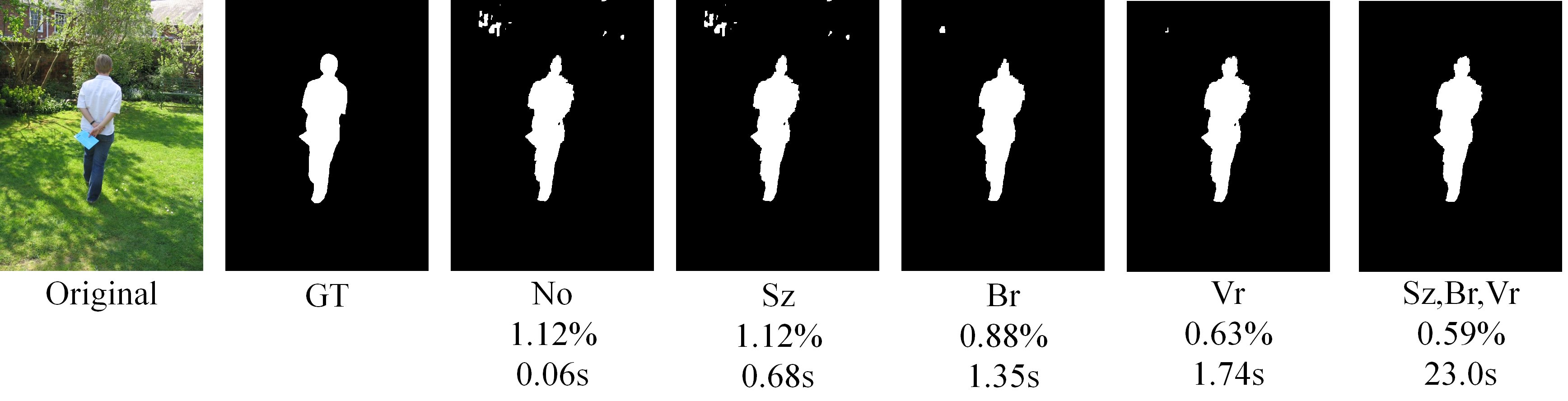} \qquad
\includegraphics[width=.47\columnwidth]{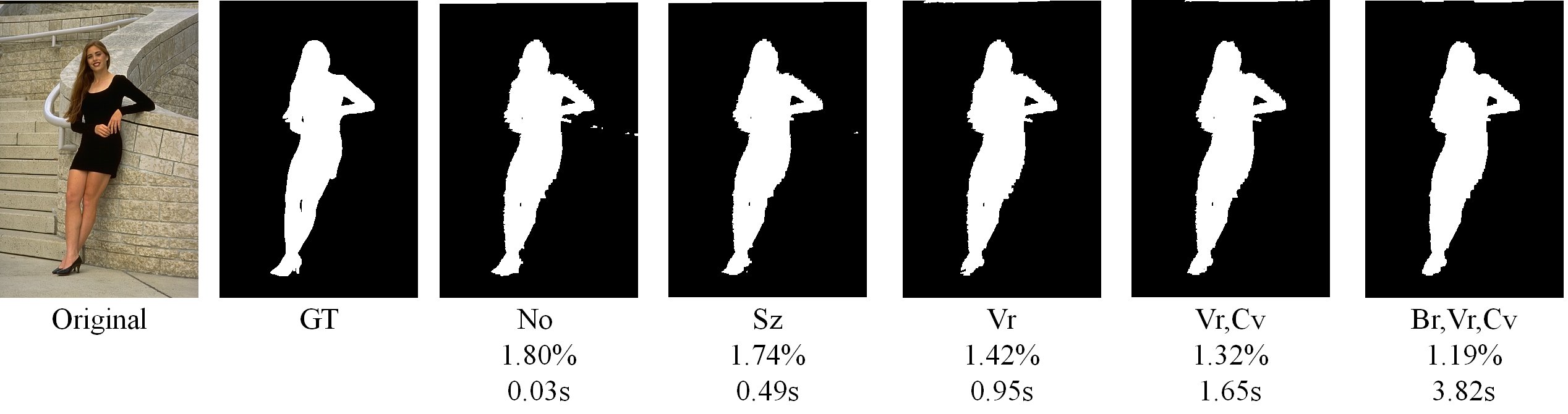}
\includegraphics[width=.47\columnwidth]{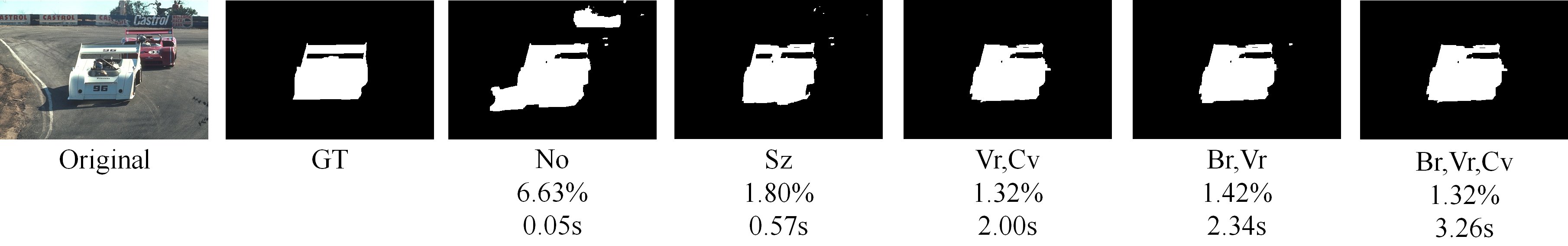} \qquad
\includegraphics[width=.47\columnwidth]{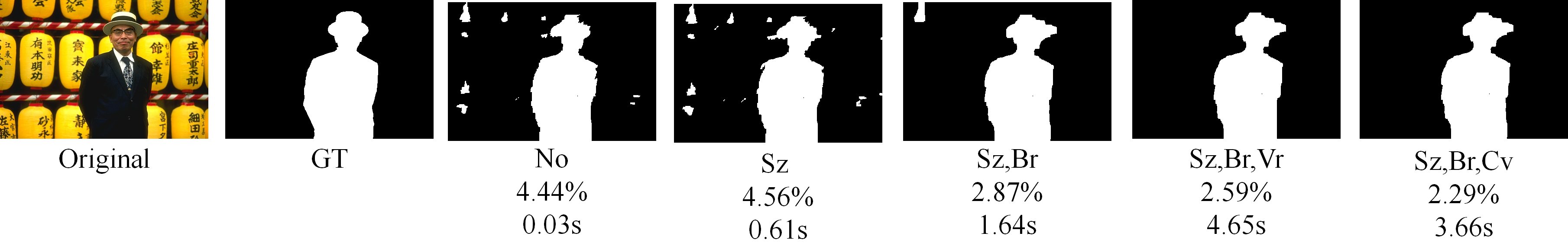}
\includegraphics[width=.47\columnwidth]{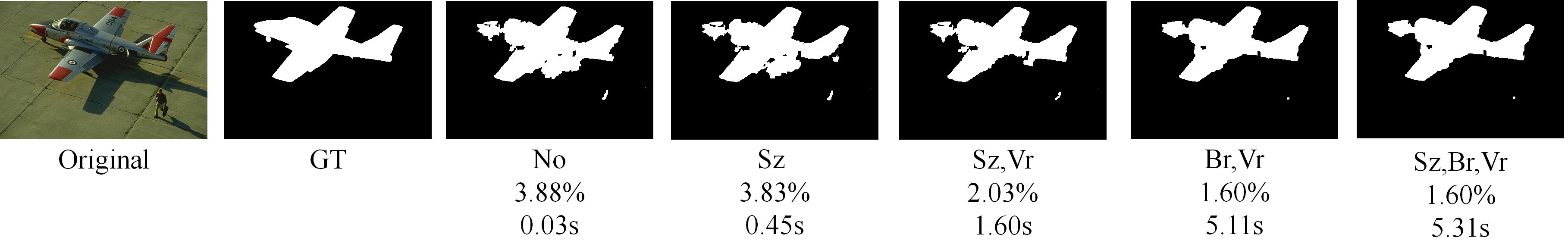} \qquad
\includegraphics[width=.47\columnwidth]{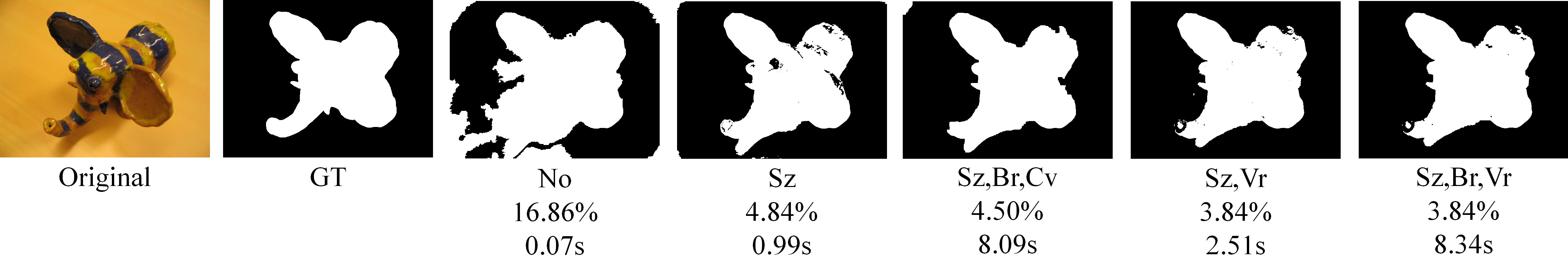}
\includegraphics[width=.47\columnwidth]{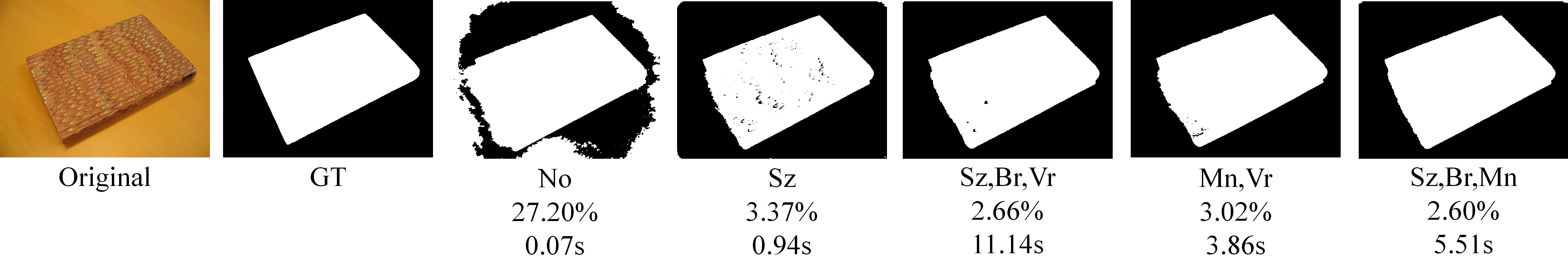} \qquad
\includegraphics[width=.47\columnwidth]{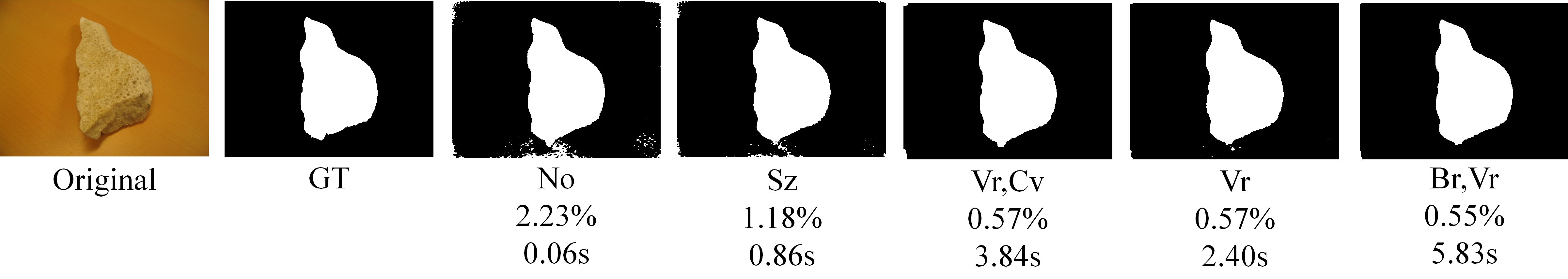}
\includegraphics[width=.47\columnwidth]{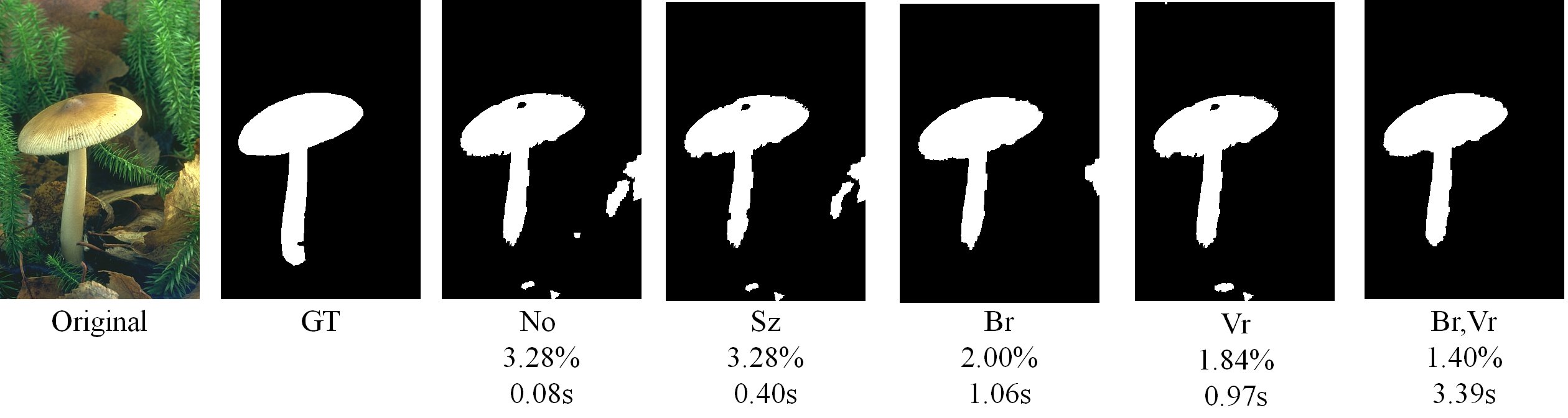} \qquad
\includegraphics[width=.47\columnwidth]{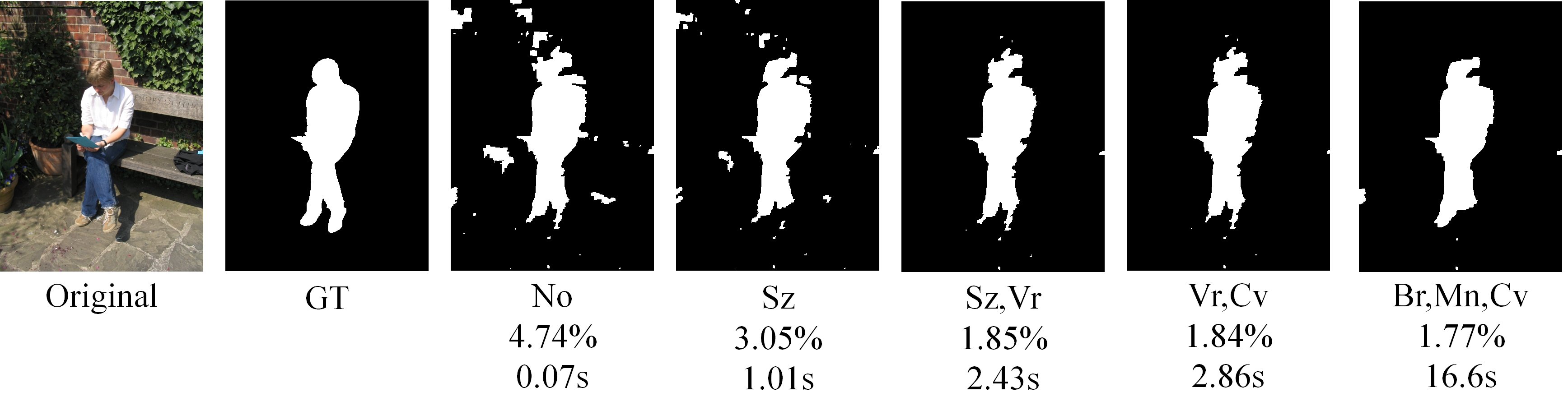}
\end{center}

\newpage

\section{Segmentation Results of Our Method, LR and QR} 

 In the following, segmentation results by our method, LR and QR are shown.
 
\begin{center}
\includegraphics[width=.47\columnwidth]{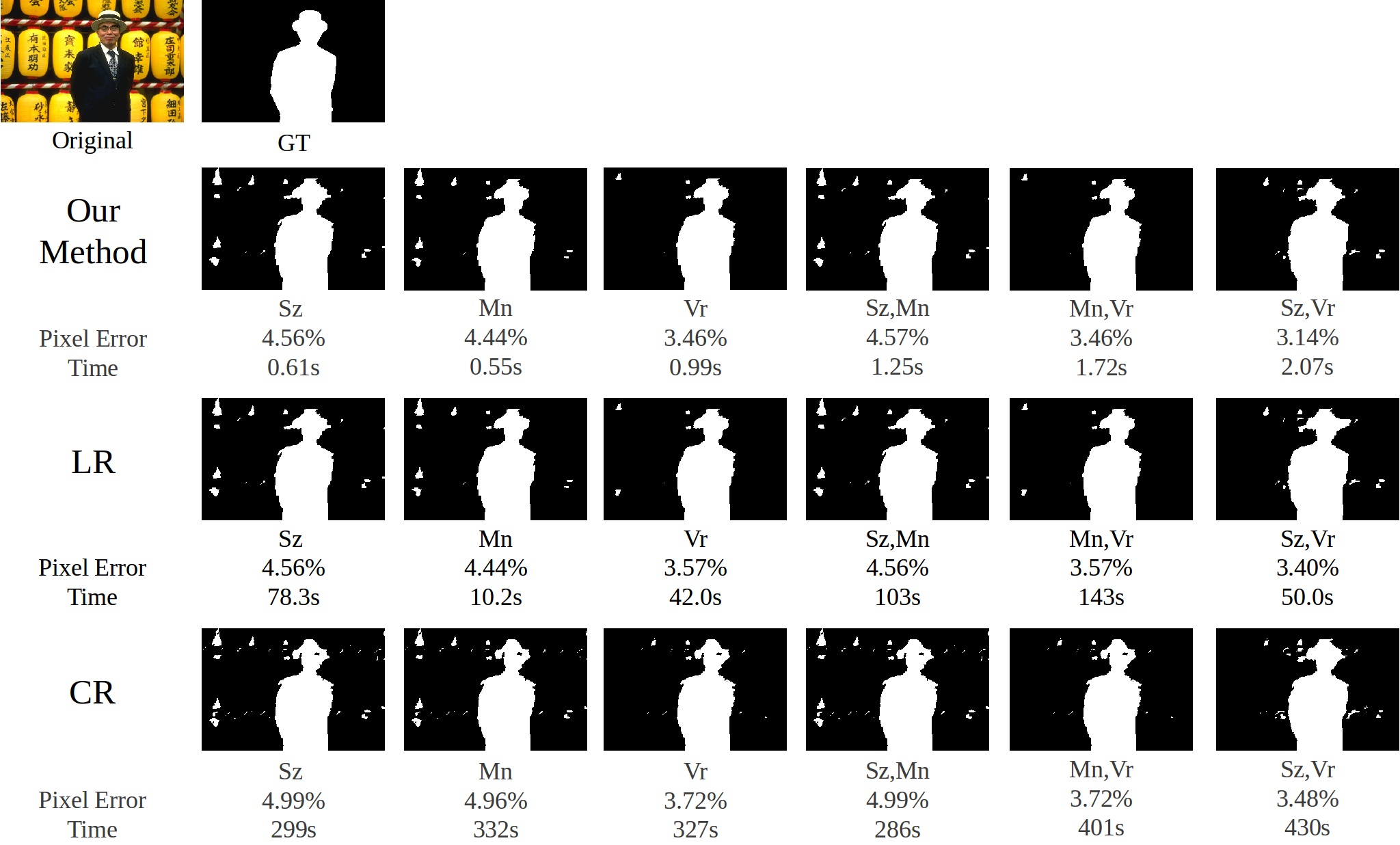} \qquad
\includegraphics[width=.47\columnwidth]{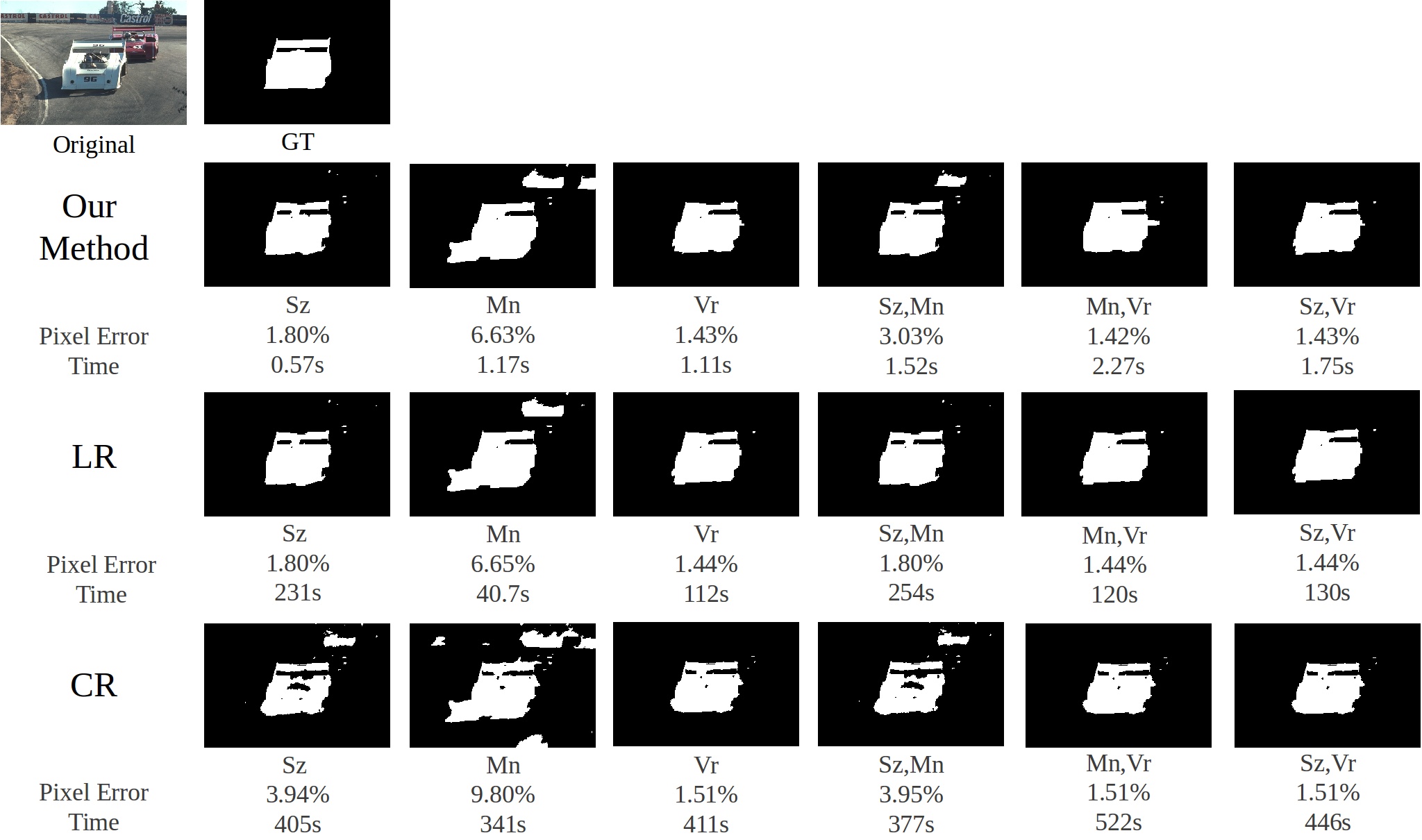}
\includegraphics[width=.47\columnwidth]{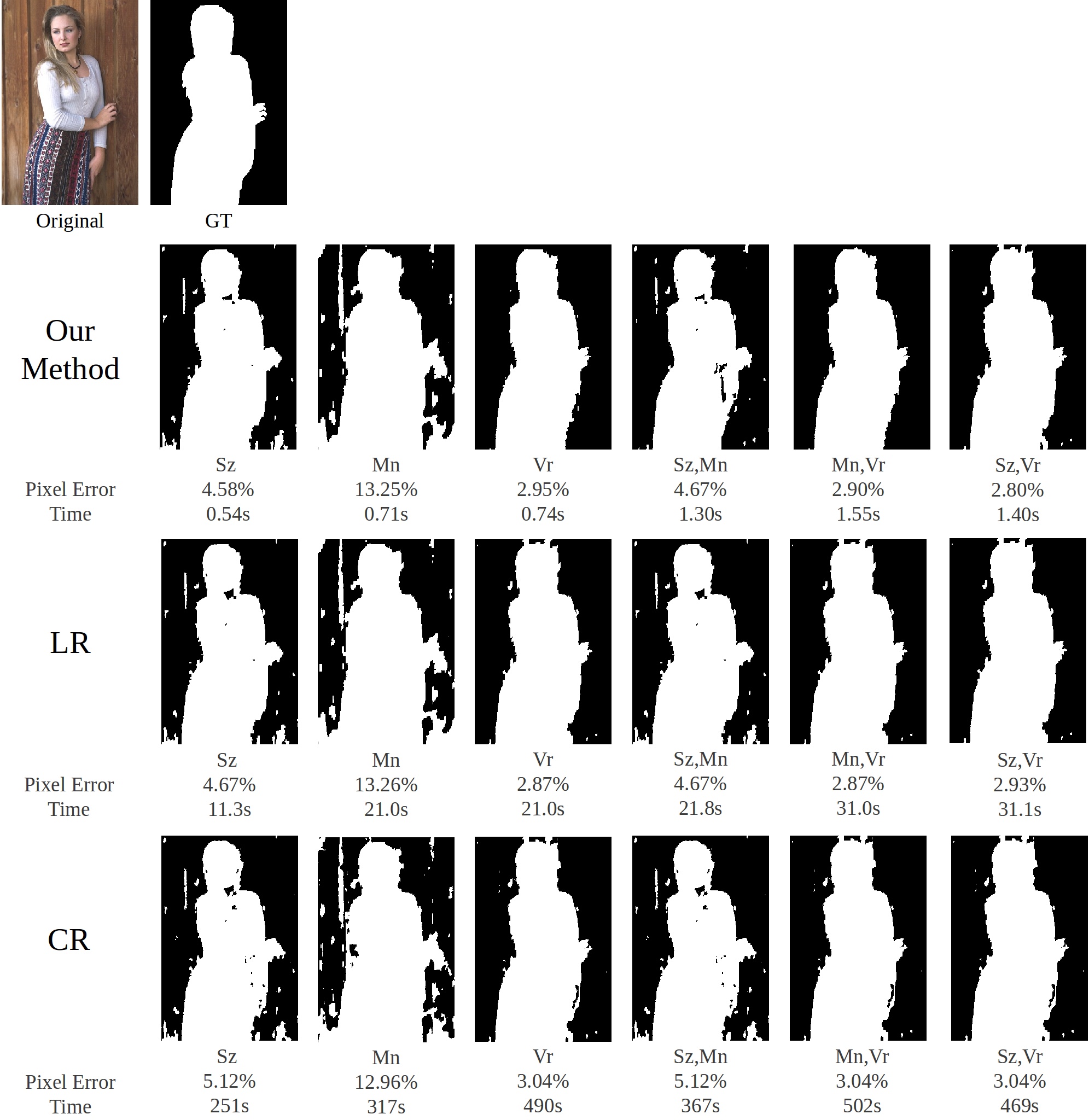} \qquad
\includegraphics[width=.47\columnwidth]{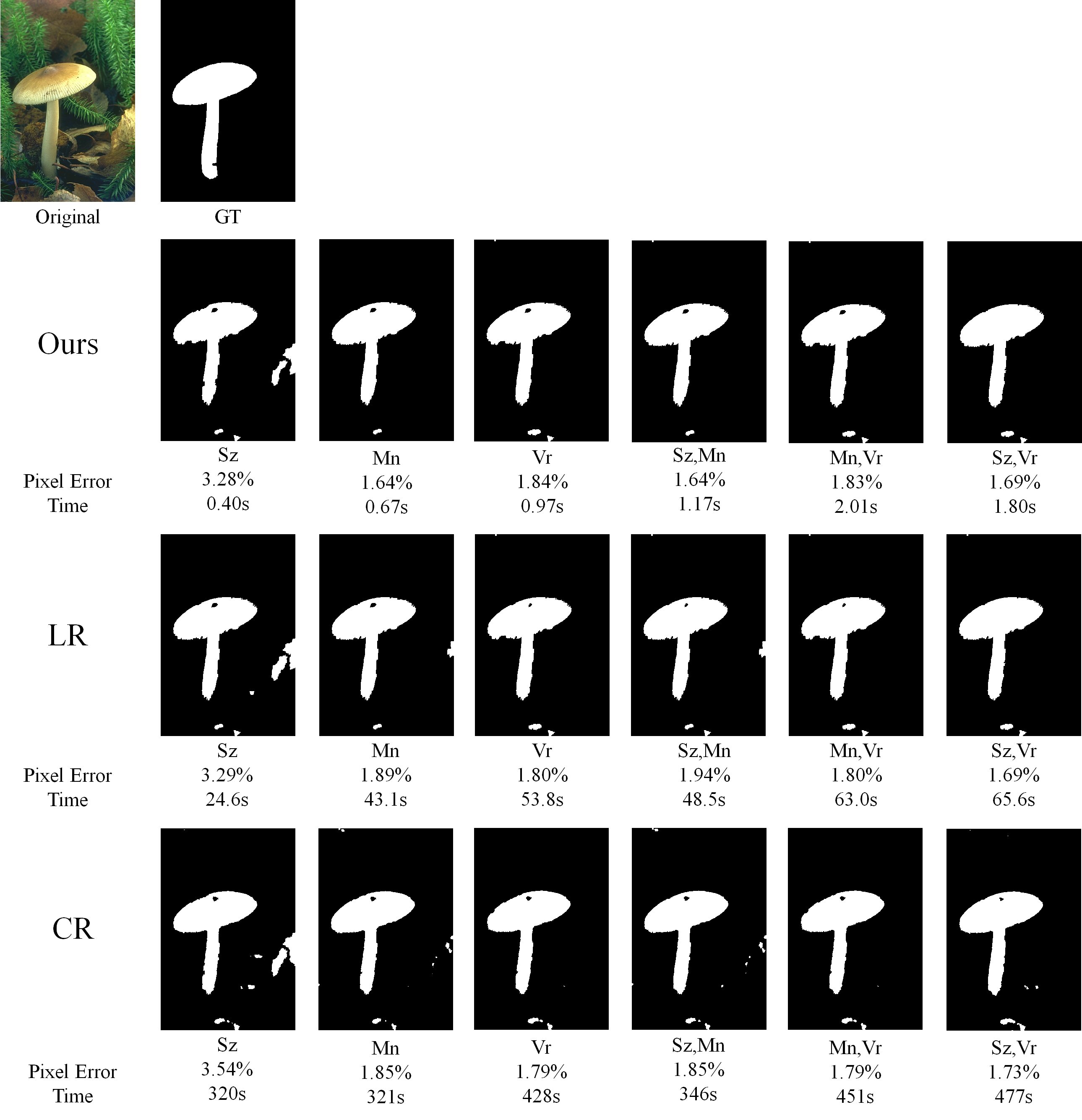}
\end{center}

\end{document}